\documentclass{article}
\usepackage{ijcai24}

\usepackage{times}
\usepackage{soul}
\usepackage{url}
\usepackage[hidelinks]{hyperref}
\usepackage[utf8]{inputenc}
\usepackage[small]{caption}
\usepackage{graphicx}
\usepackage{amsmath}
\usepackage{amssymb}
\usepackage{amsthm}
\usepackage{booktabs}
\usepackage{algorithm}
\usepackage[switch]{lineno}
\usepackage[noend]{algpseudocode}
\usepackage{multirow}
\usepackage{array}
\usepackage{xcolor}
\usepackage{colortbl}
\usepackage{bm}
\newcommand{\vz}{\mathbf{z}}
\newcommand{\vs}{\mathbf{s}}
\newcommand{\bE}{\mathbb{E}}

\newtheorem{proposition}{Proposition}
\newtheorem{assumption}{Assumption}
\newtheorem{lemma}{Lemma}
\title{Towards a Pretrained Model for Restless Bandits via Multi-arm Generalization}

\author{Yunfan Zhao\thanks{Equal contribution.}
$^1$ \and Nikhil Behari$^{*1}$ \and Edward Hughes$^{2}$ \and Edwin Zhang$^1$ \and Dheeraj Nagaraj$^2$ \and \\ Karl Tuyls$^2$ \and Aparna Taneja$^2$ \And Milind Tambe$^{1,2}$\\
\affiliations
$^1$Harvard University\\
$^2$Google   
}

\begin{document}

\maketitle

\begin{abstract}
Restless multi-arm bandits (RMABs) is a class of resource allocation problems with broad application in areas such as healthcare, online advertising, and anti-poaching. We explore several important question such as how to handle arms opting-in and opting-out over time without frequent retraining from scratch, how to deal with continuous state settings with nonlinear reward functions, which appear naturally in practical contexts. We address these questions by developing a pre-trained model (PreFeRMAB) based on a novel combination of three key ideas: (i) to enable fast generalization, we use train agents to learn from each other's experience; (ii) to accommodate streaming RMABs, we derive a new update rule for a crucial $\lambda$-network; (iii) to handle more complex continuous state settings, we design the algorithm to automatically define an abstract state based on raw observation and reward data. PreFeRMAB allows general zero-shot ability on previously unseen RMABs, and can be fine-tuned on specific instances in a more sample-efficient way than retraining from scratch.  We theoretically prove the benefits of multi-arm generalization and empirically demonstrate the advantages of our approach on several challenging, real-world inspired problems. 
\end{abstract}

\section{Introduction}
Restless multi-arm bandits (RMABs), a class of resource allocation problems involving multiple agents with a global resource constraint, have found applications in various scenarios, including resource allocation in multi-channel communication, machine maintenance, and healthcare \cite{hodge2015asymptotic,mate2022field}. RMABs have recently been studied from a multi-agent reinforcement learning perspective.  

The usual RMAB setting considers a fixed number of arms, each associated with a known, fixed MDP with finite state and action spaces; the RMAB chooses $K$ of $N$ arms every round to optimize some long term objective. Even in this setting, the problem has been shown to be PSPACE hard \cite{papadimitriou1999complexity}. Several approximation algorithms have been proposed in this setting \cite{whittle1988restless,hawkins2003langrangian}, \textit{particularly when MDP transition probabilities are fully specified,} which are successful in practice. State-of-the-art approaches for \textit{binary action} RMABs commonly provide policies based on the Whittle index \cite{whittle1988restless}, an approach that has also been generalized to \textit{multi-action} RMABs\cite{hawkins2003langrangian,killian2021beyond}. There are also linear programming-based approaches to both  \textit{binary and multi-action} RMABs \cite{zhang2021restless}. Reinforcement learning (RL) based techniques have also been proposed as state-of-the-art solutions for general \textit{multi-action RMABs}  \cite{xiong2023finite}. 

In this work, we focus on RL-based methods that provide general solutions to binary and multi-action RMABs, without requiring ground truth transition dynamics, or special properties such as indexability  as required by other approaches \cite{wang2023scalable}. Unfortunately, several limitations exist in current RMAB solutions, especially for state of the art RL-based solutions, making them challenging or inefficient to deploy in real-world resource allocation problems. 

The first limitation arises when dealing with arms that constantly opt-in (also known as streaming RMABs), which happens in public health programs where new patients (arms in RMABs) arrive asynchronously. Existing solutions either require ground truth transition probabilities \cite{mate2021efficient}, which are often unknown in practice, or else require an entirely new model to be trained repeatedly, which can be extremely computationally costly and sample inefficient. 

A second limitation occurs for new programs, or existing programs experiencing a slight change in the user base. In these situations, existing approaches do not provide a pretrained RMAB model that can be immediately deployed. In deep learning, pretrained models are the foundation for contemporary, large-scale image and text networks that generalize well across a variety of tasks \cite{bommasani2021opportunities}. For real-world problems modeled with RMABs, establishing a similar pretrained model is essential to reduce the burden of training new RMAB policies from scratch.

The third limitation occurs in handling continuous state \textit{multi-action} RMABs that have important applications \cite{sinha2022robustness,dusonchet2003continuous}. Naturally continuous domain state-spaces, such as patient adherence, are often binned into manually crafted discrete state spaces to improve model tractability and scalability \cite{mate2022field}, resulting in the loss of crucial information about raw observations.

In this work we present PreFeRMAB, a \textbf{Pre}trained \textbf{F}lexibl\textbf{e} model for \textbf{RMAB}s. Using multi-arm generalization, PreFeRMAB enables \textit{zero-shot} deployment for unseen arms as well as rapid fine-tuning for specific RMAB settings. 

{\bf Our main contributions are:}
\begin{itemize}
\item To the best of our knowledge, we are the first to develop a pretrained RMAB model with zero shot ability on entire sets of unseen arms. 

\item Whereas a general multiagent RL system could suffer from sample complexity exponential in the number of agents $N$ \cite{gheshlaghi2013minimax}, we prove PreFeRMAB benefits from larger $N$, via multi-arm generalization and better estimation of the population distribution of arm features. 


\item Our pretrained model can be fine-tuned on specific instances in a more sample-efficient way than training from scratch, requiring less than 12.5\% of samples needed for training a previous \textit{multi-action} RMAB model in a healthcare setting \cite{verma2023restless}.

\item We derive an update rule for a crucial $\lambda$-network, allowing changing numbers of arms without retraining. While streaming bandits received considerable attention \cite{liau2018stochastic}, we are the first to handle streaming \textit{multi-action} RMABs with unknown transition dynamics. 

\item Our model accommodates both discrete and continuous states. To address the continuous state setting, where real-world problems often require nonlinear rewards \cite{riquelme2018deep}, we providing a StateShaping module to automatically define an abstract state.

\end{itemize}

\section{Related Work}
{\bf RMABs with binary and multiple actions.} \ Solving an RMAB problem, even with known transition dynamics, is known to be PSPACE hard \cite {papadimitriou1999complexity}. For binary action RMABs, \cite{whittle1988restless} provides an approximate solution, using a Lagrangian relaxation to decouple arms and choose actions by computing so-called Whittle indices of each arm. It has been shown that the Whittle index policy is asymptotically optimal under the indexability condition \cite{weber1990index,akbarzadeh2019restless}. The Whittle index was extended to a special class of \textit{multi-action} RMABs with monotonic structure \cite{hodge2015asymptotic}. A method for more general \textit{multi-action} RMABs based on Lagrangian relaxation was proposed by \cite{killian2021beyond}. Weakly coupled Markov Decisions Processes (WCMDP), which generalizes multi-action RMABs to have multiple constraints, was studied by \cite{hawkins2003langrangian}, who proposed a Langrangian decomposition approach. WCMDP was subsequently studied by \cite{adelman2008relaxations}, who proposed improvements in solution quality at the expense of higher computational costs. While above methods developed for WCMDP require knowledge of ground truth transition dynamics, our algorithm handles unknown transition dynamics, which is more common in practice\cite{wang2023scalable}. Additionally, the above works in \textit{multi-action} settings do not provide algorithms for continuous state RMABs. 

{\bf Multi-agent RL and RL for RMABs.} \  RMABs are a specific instance of the powerful multi-agent RL framework used to model systems with multiple interacting agents in both competitive and co-operative settings \cite{shapley1953stochastic,littman1994markov}, for which significant strides have been made empirically \cite{jaques2019social,yu2022surprising} and theoretically \cite{jin2021v,xie2020learning}. \cite{nakhleh2021neurwin} proposed a deep RL method to estimate the Whittle index. \cite{fu2019towards} provided an algorithm to learn a Q-function based on  Whittle indices, states, and actions. \cite{avrachenkov2022whittle} and \cite{biswas2021learn} developed Whittle index-based Q-learning methods with convergence guarantees. While the aforementioned works focus on binary action RMABs, \cite{killian2021q} generalized this to multi-action RMABs using tabular Q-learning. A subsequent work \cite{killian2022restless}, which focussed on robustness against adversarial distributions, took a deep RL approach that was more scalable. However, existing works on multi-action RMABs do not consider streaming RMABs and require training from scratch when a new arm opts-in. Additionally, works built on tabular Q-learning \cite{fu2019towards,avrachenkov2022whittle,biswas2021learn,killian2021q} may not generalize to continuous state RMABs without significant modifications. Our pretrained model addresses these limitations, and enables zero-shot ability on a wide range of unseen RMABs. 

{\bf Streaming algorithms.} \ The streaming model, pioneered by \cite{alon1996space}, considers a scenario where data arrives online and the amount of memory is limited. The model is adapted to multi-arm bandits (MAB), assuming that arms arrive in a stream and the number of arms that can be stored is limited \cite{liau2018stochastic,chaudhuri2020regret}. The streaming model has recently been adapted to \textit{binary action} RMABs with \textit{known} transition probabilities\cite{mate2021efficient}, but not studied in the more general and practical settings of \textit{multi-action} RMABs with \textit{unknown} transition dynamics. We aim to close this gap.  

{\bf Zero-shot generalization and fine-tuning.} Foundation models that have a strong ability to generalize to new tasks in zero shot and efficiently adapt to new tasks via fine-tuning have received great research attention \cite{bommasani2021opportunities}. Such models are typically trained on vast data, such as internet-scale text \cite{devlin2018bert} or images \cite{ramesh2021zero}. RL has seen success in the direction of foundation models for decision making, using simulated \cite{team2023human} and real-world \cite{yu2020learning} environments. A pretrained model for RMABs is needed \cite{zhao2023towards}. To our knowledge, we are the first to realize zero-shot generalization and efficient fine-tuning in the setting of RMABs.  

\section{Problem Statement}
We study multi-action RMABs with system capacity $N$, where existing arms have the option to opt-out (that is, the state-action-rewards corresponding to them are disregarded by the model post opt-out), and new, unseen arms can request to opt-in (that is, these arms are considered only post the opt-in time). Such requests will be accepted if and only if the system capacity permits. 
A vector $\xi_t\in\{0,1\}^N$ represents the opt-in decisions:

\begin{align*}
\xi_{i,t} =
\begin{cases}
    1& \text{if arm $i$ opts-in at round $t$,}\\
    0              & \text{otherwise.}
\end{cases}
\end{align*}

Notice that existing arms must opt-in in each round $t$ to remain in the system. For each arm $i\in[N]$, the state space $\mathcal{S}_i$ can be either discrete or continuous, and the action space $\mathcal{A}_i$ is a finite set of discrete actions. Each action $a\in\mathcal{A}_i$ has an associated cost $\mathcal{C}_i(a)$, with $\mathcal{C}_i(0)$ denoting a no-cost passive action. The reward at a state is given by a function $R_i: \mathcal{S}_i \rightarrow\mathbb{R}$. We let $\beta\in[0,1)$ denote a discount factor. Each arm has a unique feature vector $\bm z_i\in\mathbb{R}^m$ that provides useful information about the arm.  Notice our model directly utilizes feature information in its policy network, without requiring intermediate steps to extract transition dynamics information from features. 

When the state space is discrete, each arm $i\in[N]$ follows a Markov Decision Process $(\mathcal{S}_i, \mathcal{A}_i, \mathcal{C}_i, T_i, R_i,\beta, \bm z_i)$, where $T_i: \mathcal{S}_i\times \mathcal{A}_i \times \mathcal{S}_i \rightarrow[0,1]$ is a transition matrix representing the probability of transitioning from the current state to the next state given an action.  In contrast, when the state space is continuous, each arm $i\in[N]$ follows a Markov Decision Process $(\mathcal{S}_i, \mathcal{A}_i, \mathcal{C}_i, \bm\Gamma_i, R_i,\beta, \bm z_i)$, where $\bm\Gamma_i$ is a set of parameters encoding the transition dynamics. For example, in the case that the next state moves according to a Gaussian distribution, $\bm\Gamma_i$ may denote the mean and variance of the Gaussian. 

For simplicity, we assume that $\mathcal{S}_i, \mathcal{A}_i, \mathcal{C}_i$, and $R_i$ are the same for all arms $i\in[N]$ and omit the subscript $i$. Note that our algorithms can also be used in the general case where rewards and action costs are different across arms. For ease of notation, we let $\bm s\in\mathbb{R}^N$ denote the state over all arms, and we let $\bm A\in\{0,1\}^{N\times |\mathcal{A}|}$ denote one-hot-encoding of the actions taken over all arms. The agent learns a policy $\pi$ that maps states $\bm s$ and features $\bm z$ to actions $\bm A$, while satisfying a constraint that the sum cost of actions taken is no greater than a given budget $B$ in every timestep $t\in[H]$, where $H$ is the length of the horizon.

{\bf Our goal is to learn an RMAB policy that maximizes the following Bellman equation} The key difficulty in learning such a policy is how to utilize features $\bm z$ and address opt-in decisions $\bm \xi$. These are important research questions not addressed in previous works \cite{killian2022restless}. 
\begin{align}
\label{eq:original_opti_problem}
&J(\boldsymbol{s}, \bm z, \bm \xi)=\max _{\boldsymbol{A}}\left\{\sum_{i=1}^N R\left(\boldsymbol{s}_i\right)+\beta \, \mathbb{E}\left[J\left(\boldsymbol{s}^{\prime}, \bm z, \bm \xi\right) \mid \boldsymbol{s}, \boldsymbol{A}\right]\right\} \, ,\\
&\text { s.t. } \sum_{i=1}^N \sum_{j=1}^{|\mathcal{A}|} \boldsymbol{A}_{i j} c_j \leq B \quad \text{and} \quad \sum_{j=1}^{|\mathcal{A}|} \boldsymbol{A}_{i j}=1 \quad \forall i \in[N]\, ,\nonumber
\end{align}
where $c_j \in \mathcal{C}$ is the cost of $j^{\text{th}}$ action, and $A_{ij} = 1$ if action $j$ is chosen on arm $i$ and $A_{ij} = 0$ otherwise. Further, we assume that the rewards $R$ are uniformly bounded by $R_{\max}$. 


\section{Generalized Model for RMABs}

We first provide an overview of key ideas and then discuss each of the ideas in more detail. 
(See Figure~\ref{fig:overview} in Appendix for an overview of the training procedure.)


\subsection{Key Algorithmic Ideas}
Several key algorithmic novelties are necessary for our model to address limitations of existing works:

{\bf A pretrained model via multi-arm generalization: } We train agents to learn from each others' experience. Whereas a general multiagent RL system could suffer from sample complexity exponential in the number of arms $N$ \cite{gheshlaghi2013minimax}, we prove PreFeRMAB benefits from a larger $N$, via generalization across arms.


{\bf A novel $\lambda$-network updating rule for opt-in: } 
The opt-in and opt-out of arms induce a more complex form of the Lagrangian and add randomness to actions taken by agents. We provide a new $\lambda$-network update rule and train PreFeRMAB with opt-in and opt-out of arms, to enable zero-shot performance across various opt-in rates and accommodate streaming RMABs.

{\bf Handling continuous states with StateShaping subroutine:} 
In the continuous state setting, real-world problems often require nonlinear rewards \cite{riquelme2018deep}, and naively using raw observations to train models may result in poor performance (see Table~\ref{table:state_shaping}). To tackle this challenge, we design the algorithm to automatically define an abstract state based on raw observation and reward data. 


\subsection{A Pretrained Model via Multi-arm Generalization}
To enable multi-arm generalization, we introduce feature-based Q-values, together with a Lagrangian relaxation with features $\bm z_i$ and opt-in decisions $\xi_i$:
\begin{align}
\label{eq:lagrangian_relaxation}
&J\left(s, \bm z, \bm \xi, \lambda^{\star}\right) \nonumber\\
= \ &\min _{\lambda\geq 0}\left(\frac{\lambda B}{1-\beta}+\sum_{i=1}^N \max _{a_i \in|\mathcal{A}|}\left\{Q\left(\boldsymbol{s}_i, a_i, \bm z_i, \xi_i, \lambda\right)\right\}\right),\\
\text{s.t.} \ &Q\left(\boldsymbol{s}_i, a_i, \bm z_i, \xi_i, \lambda\right)\nonumber \\
= \ &\xi_i R\left(\boldsymbol{s}_i\right) - \xi_i \lambda c_{a_i} + \beta\, \mathbb{E}\left[Q\left(\boldsymbol{s}_i^{\prime}, a_i, \bm z_i, \xi_i, \lambda\right) \mid \pi(\lambda)\right] .\nonumber
\end{align}

where $Q$ is the Q-function, $a_{i}$ is the action of arm $i$, $s_i'$ is the state transitioned to from $s_i$ under action $a_{i}$, and $\pi(\lambda)$ is the optimal policy under a given $\lambda$. Notice that this relaxation decouples the Q-functions of the arms, and therefore $Q_i$ can be solved independently for a given $\lambda$. 

Now we discuss how we use feature-based Q-values and how agents could learn from each other. During pretraining, having decided which arms opt-in and out (lines 5-8), Algorithm~\ref{alg:main_algorithm_training} samples an action-charge $\lambda$ based on updated opt-in decisions $\xi$ and features $\bm z_i$ (lines 9-11). Next, from opt-in arms we collect trajectories (lines 12-19), which are later used to train a single pair of actor/critic networks for all arms, allowing the policy for one arm to benefit from other arms' data. After that, we update the policy network $\theta$ and the critic network $\phi$ (Line 21), using feature-based Q-values to compute advantage estimates for the actor in PPO update. Critically, feature-based Q-values updated with one arm's data, improves the policy for other arms. In real-world problems with missing feature entries or less informative features, it is more important for agents to learn from each other (see Table~\ref{table:uniquesamples_vs_reward} in Sec~\ref{sec:exp_discrete_state}). Intuitively, if a model only learns from homogeneous arms, then we should expect this model to perform poorly when used out-of-the-box on arms with completely different behaviors. 

\begin{algorithm}[!h]
\caption{PreFeRMAB (Training)}
\begin{algorithmic}[1]
\State Input: n\_epochs, n\_steps, $\lambda$-update frequency $K\in\mathbb{N}^+$, and system capacity $N$
\State Initialize policy network $\theta$, critic network $\phi$, and $\lambda$-network $\Lambda$, buffer = [], and state $\bm s\in\mathbb{R}^N$, features $\bm z_i\in\mathbb{R}^m$ 
\State Initialize \textit{StateShaping}, and set $\bar{\bm s} \leftarrow$\textit{StateShaping}$(\bm s)$

\For{epoch = 1, 2, . . . , n\_epochs}
    \State Current beneficiaries submit opt-in requests $\xi_i\in\{0,1\}$
    \State New beneficiaries submit opt-in requests
    \State Accept the first $N$ requests
    \State Add dummy states / features for any unused capacity
    \State Update the vector $\xi$ encoding opt-in decisions
    \State Update features $\bm z_i$ based on new beneficiaries' features
    \State Compute $\lambda = \Lambda(\bar{\bm s},\{z_i\}_{i=1}^N, \bm\xi)$
    \For{timestep $t$ = 1, . . . , n\_steps}
        \For {Arm $i$ = 1, . . . , N} 
            \If{Arm $i$ is opt-in (i.e. $\xi_i=1$)}
                \State Sample an action $a_i\sim\theta(\bar{s}_i,\lambda, \bm z_i)$
                \State $s'_i,  r_i = \text{Simulate}(s_i,  a_i)$
                \State $\bar{s'_i}=$ \textit{StateShaping}$(s'_i)$
                \State Add tuple $(s_i, \bar{s_i},  a_i, r_i, \bar{s'_i},\bm z_i)$ to buffer
                \State $s_i\leftarrow s'_i$, $\bar{ s_i}\leftarrow \bar{s'_i}$
            \EndIf
        \EndFor                
    \EndFor
    \State Add tuple $(\lambda, \bm\xi)$ to buffer
    \State Update the $(\theta,\phi)$ pair via PPO, using trajectories in buffer 
    \If {epoch // K = 0}
        \State Update $\Lambda$ via Prop~\ref{prop:lambda_updating_rule} using trajectories in buffer
        \State Update $\hat r(\cdot)$ in \textit{StateShaping} using $(s,r)$-tuples in buffer 
    \EndIf
\EndFor 
\end{algorithmic}
\label{alg:main_algorithm_training}
\end{algorithm}


We will now give a theoretical guarantee of multi-arm generalization by considering the following simplified setting and assumptions where we do not consider opt-in and opt-out. That is, in each epoch we draw a new set of $N$ arms. We let the distribution of arm features (i.e, $\vz_i$) to be denoted by the probability measure $\mu^{*}$. Each `sample' for our policy network training consists of $N$ features corresponding to $N$ arms $(\vz_1,\dots,\vz_N)$, drawn i.i.d. from the distribution $\mu^{*}$. Call the empirical distribution of $(\vz_i)$ to be $\hat{\mu}$. During training, we receive $n_{\mathsf{epochs}}$ i.i.d. draws of $N$ arm features each, denoted by $\hat{\mu}_1,\dots,\hat{\mu}_{n_{\mathsf{epochs}}}$ 

Let $\Theta$ denote the space of neural network weights of the policy network (for clarity, we shorten the $(\theta,\phi)$ in Algorithm~\ref{alg:main_algorithm_training} to $\theta$). The neural network inputs are Lagrangian multiplier $\lambda$, state of an arm $s$, its feature $\vz$ and the output is $a \in \mathcal{A}$. Let $V(\vs,\theta,\lambda,\hat{\mu}) $ denote the discounted reward, averaged over $N$ arms with features $\hat{\mu}$ obtained with the neural network with parameter $\theta$, starting from the state $\vs$ (cumulative state of all arms). The proposition below shows the generalization properties of the output of Algorithm~\ref{alg:main_algorithm_training}. The proof and a detailed discussion of the assumptions and consequences are given in Section~\ref{sec:multiarm_gen_proof}.

\begin{proposition}\label{prop:multiarm_gen}
Suppose the following assumptions hold:
\begin{enumerate}
\item Algorithm~\ref{alg:main_algorithm_training} learns neural network weights $\hat{\theta} \in \Theta$, whose policy is optimal for each $(\hat{\mu}_i,\lambda)$ for $1\leq i \leq n_{\mathsf{epochs}}$ and $\lambda \in [0,\lambda_{\max}]$
\item There exists $\theta^{*} \in \Theta$ which is optimal for every instance $(\hat{\mu},\lambda)$.
\item  $\Theta = \mathcal{B}_2(D,\mathbb{R}^d)$, the $\ell_2$ ball of radius $D$ in $\mathbb{R}^d$.
\item $|V(\vs,\theta_1,\lambda,\hat{\mu})-V(\vs,\theta_2,\lambda,\hat{\mu})|\leq L\|\theta_1-\theta_2\|$ and $|V(\vs,\theta,\lambda_1,\hat{\mu})-V(\vs,\theta,\lambda_2,\hat{\mu})|\leq L|\lambda_1-\lambda_2|$ for all $\theta_1,\theta_2,\theta \in \Theta$ and $\lambda_1,\lambda_2,\lambda \in [0,\lambda_{\max}]$.
\end{enumerate}
       
    Then, the generalization error over unseen arms ($\hat{\mu}$) satisfies:
\begin{align}
\bE_{\hat{\mu},\hat{\theta}}[\inf_{\lambda \in [0,\lambda_{\max}]}V(\vs,\hat{\theta},\lambda,\hat{\mu})] &\geq \bE_{\hat{\mu}}[\inf_{\lambda \in [0,\lambda_{\max}]}V(\vs,\theta^{*},\lambda,\hat{\mu})] \nonumber\\ &\quad - \tilde{O}\left(\tfrac{1}{\sqrt{n_{\mathsf{epochs}}N}}\right)\end{align}
Here, $\tilde{O}$ hides polylogarithmic factors in $n_{\mathsf{epochs}},N$ and constants depending on $d,D,L,\beta,\frac{B}{N},c_j,R_{\max}$ and $\lambda_{\max}$
\end{proposition}
The assumption of existence of $\theta^*$ is reasonable: This means that there exists a neural network which gives the optimal policy for a family of single-arm MDPs indexed by $(\vz,\lambda)$. Proposition~\ref{prop:multiarm_gen} shows that when $n_{\mathsf{epochs}}$ and $N$ are large, the Lagrangian relaxed value function of the learned network is close to that of the optimal network. 

\textit{An important insight is that the generalization ability of the PreFeRMAB network becomes better as the number of arms per instance becomes larger.} This is counter intuitive since a system with a larger number of agents are generally very complex. Jointly, the arms form an MDP with $|\mathcal{S}|^N$ states and $|\mathcal{A}|^{N}$ actions. General multi-agent RL problems with $N$ arms thus can suffer from an exponential dependence on $N$ in their sample complexity for learning (see sample complexity lower bounds in \cite{gheshlaghi2013minimax}). However, due to the structure of RMABs and the Lagrangian relaxation, we achieve a better generalization with a larger $N$. Our proof in the appendix shows that this is due to the fact that a larger number of arms helps estimate the population distribution $\mu^*$ of the arm features better. We show in Table~\ref{table:uniquesamples_vs_reward} that indeed having more number of arms helps the PreFeRMAB network generalize better over unseen instances. 
\subsection{A Novel $\lambda$-network Updating Rule}

In real-world health programs, we may observe new patients constantly opt-in \cite{mate2021efficient}. The opt-in / opt-out decisions render the updating rule in \cite{killian2022restless} unusable and add additional randomness to actions taken by the agent. To overcome this challenge and to stabilize training, we develop a new $\lambda$-network updating rule. 

\begin{proposition}\label{prop:lambda_updating_rule}[$\lambda$-network updating rule]
The equation for gradient descent for the objective (Eq~\ref{eq:lagrangian_relaxation}) with respect to $\lambda$, with step size $\alpha$ is:
\begin{align*}
\Lambda_t =& \Lambda_{t-1} - \alpha\left(\frac{B}{1-\beta}\right) \\
& - \alpha\left(\sum_{i=1}^N  \mathbb{E}\left[\sum_{t=0}^H \xi_{i,t} \beta^t c_{i,t} + (1-\xi_{i,t})\beta^tc_{0,t}
\right]\right),
\end{align*}
where $c_{i,t}$ is the cost of the
action taken by the optimal policy on arm $i$ in round $t$. 
\end{proposition}

Critically, this update rule allows PreFeRMAB to handle streaming RMABs, accommodating a
changing number of arms without retraining and achieving strong zero-shot performance across various opt-in rates (see Table~\ref{table:dist-shift-sis} and \ref{table:discrete_state_armman}). Having established an updating rule, we provide a convergence guarantee. The proofs are relegated to Appendix~\ref{sec:appendix_proofs}.

\begin{proposition}[Convergence of $\lambda$-network]
\label{prop:convergence_lambda}
Suppose the arm policies converge to the optimal $Q$-function for a given $\Lambda_t$, then the update rule (in Prop~\ref{prop:lambda_updating_rule}) for the $\lambda$-network converges to the optimal as the number of training epochs and the number of actions collected in each epoch go to infinity.
\end{proposition}

\subsection{Handling Continuous States with StateShaping}

Real-world problems may require continuous states with nonlinear rewards \cite{riquelme2018deep}. Existing RMAB algorithms either use a human-crafted discretization or fail to address challenging nonlinear rewards \cite{killian2022restless}. Discretization may result in loss of information and fail to generalize to different population sizes. For example, the popular SIS epidemic model \cite{yaesoubi2011generalized} is expected to scale to a continuum limit as the population size increases to infinity, and a continuous state-space model can better handle scaling by using proportions instead of absolute numbers. Under nonlinear rewards, naively using raw observations in training may result in poor performance (see Table~\ref{table:state_shaping}). We provide a StateShaping module to improve model stability and performance. 


\begin{algorithm}[!h]
\caption{StateShaping Subroutine}
\begin{algorithmic}[1]
\State Input: estimator $\in\{\text{Isotonic Regression}, \text{KNN}\}$, states \( \bm s \in \mathbb{R}^N \), data \( \mathcal{D} \) of \( (s,r) \) tuples 
\State Output \( \bar{\bm s} = \bm s \) if no normalization is desired
\State Compute \vspace{-2mm}
\begin{align*}&r_{min} = \min_{s': \ s'\in\mathcal{D}} r(s'), \quad r_{max} = \max_{s': \ s'\in\mathcal{D}} r(s')\\
& s_{min} = \min_{s'\in\mathcal{D}} s', \quad s_{max} = \max_{s'\in\mathcal{D}} s'\end{align*}

\State Compute $\hat{r}(s_i)$ using the choice of Estimator.
    
\State Output \( \bar{\bm s} \), where \  \( \bar{s}_i = \frac{\hat{r}(s_i) - r_{\min}}{r_{\max} - r_{\min}} (s_{\max} - s_{\min}), \forall i \)
\end{algorithmic}
\label{alg:state_shaping}
\end{algorithm}



    

In Algorithm~\ref{alg:state_shaping}, users can choose whether to obtain abstract state \cite{abel2018state} from raw observations (lines 2). We compute ranges of reward and raw observations, and obtain an reward estimate (lines 3-4). After that, we automatically refine the raw observation such that reward is a linear function of the abstract state (line 4), improving model stability for challenging reward functions. Here a key assumption is that reward is an increasing function of the raw observation, which is common in RMABs \cite{killian2022restless}. Notice as we collect more observations, the accuracy of the reward estimator $\hat r(\cdot)$ will improve (it is updated in line 24 of Algorithm~\ref{alg:main_algorithm_training}).



\subsection{Inference using Pretrained Model}

An important difference between training and inference is that during inference time, we strictly enforce the budget constraint on the trained model, by greedily selecting highest probability actions until the budget is reached. The rest of the inference components are similar to the training component. 



\begin{algorithm}[!h]
\caption{PreFeRMAB (Inference)}
\begin{algorithmic}[1]
\State Input: : States $\bm s$, costs $C$, budget $B$, features $\bm z_i \in\mathbb{R}^m$, opt-in decisions $\bm\xi$, agent actor $\theta$, $\lambda$-network, \textit{StateShaping} routine with trained estimator $\hat{r}(\cdot)$ 
\State Compute $\lambda = \Lambda(\bar{\bm s},\{z_i\}_{i=1}^N, \bm\xi)$
\For {Arm $i$ = 1, . . . , N} 
    \If{Arm $i$ is opt-in (i.e. $\xi_i=1$)}
        \State $\bar{s_i}=$ \textit{StateShaping}$(s_i)$
        \State Compute $p_i\sim\theta(\bar{s}_i,\lambda, \bm z_i)$
    \EndIf
\EndFor
\State $\bm a$ = GreedyProba$(p, C, B)$ \hfill\Comment{Greedily select highest probability actions until budget B is reached}
\end{algorithmic}
\label{alg:main_algorithm_inference}
\end{algorithm}


\section{Experimental Evaluation}
We provide experimental evaluations of our model in \textit{three separate domains}, including a synthetic setting, an epidemic modeling setting, as well as a maternal healthcare intervention setting. We first describe these three experimental domains. Then, we provide results for PreFeRMAB in a \textit{zero-shot evaluation setting}, demonstrating the performance of our model on \textit{new, unseen test arms} drawn from distributions distinct from those in training. 
Here, we demonstrate the flexibility of PreFeRMAB, including strong performance across domains, state representations (discrete vs. continuous), and over various challenging reward functions. Finally, we demonstrate the strength of using PreFeRMAB as a \textit{pre-trained model}, enabling faster convergence for fine-tuning on a specific set of evaluation arms. 

In Appendix~\ref{sec:appendix_ablation}, we provide {\bf ablation studies} over (1) a wider range of opt-in rates (2) different feature mappings (3) DDLPO topline with and without features (4) more problem settings. 

\subsection{Experimental Settings}
\textbf{Features:} In all experiments, we generate features by projecting parameters that describe the ground truth transition dynamics into features using randomly generated projection matrices. The dimension of feature equals the number of parameters required to describe the transition dynamics. In Appendix~\ref{sec:appendix_ablation}, we provide results on different feature mappings. 

\textbf{Synthetic:} Following \cite{killian2022restless}, we consider a synthetic dataset with binary states and binary actions. The transition probabilities for each arm $i$ are represented by matrices $T_{s=0}^{(i)}$ and $T_{s=1}^{(i)}$ for arm $i$ at states $0$ and $1$ respectively:
\[ 
T_{s=0}^{(i)} = 
\begin{bmatrix}
p_{00} & 1 - p_{00} \\
p_{01} & 1 - p_{01}
\end{bmatrix}
, \ \ 
T_{s=1}^{(i)} = 
\begin{bmatrix}
p_{10} & 1 - p_{10} \\
p_{11} & 1 - p_{11}
\end{bmatrix}
\]

Each $p_{jk}$ corresponds to the probability of transitioning from state $j$ to state 0 when action $k$ is taken. These values are sampled uniformly from the intervals:
\[
p_{00} \in [0.4, 0.6], p_{01} \in [0.4, 0.6], p_{10} \in [0.8, 1], p_{11} \in [0.0, 1] 
\]

{\bf SIS Epidemic Model:} Inspired by the vast literature on agent-based epidemic modeling, we adapt the SIS model given in \cite{yaesoubi2011generalized}, following a similar experiment setup as described in \cite{killian2022restless}. Arms $p$ represent a subpopulation in distinct geographic regions; states $s$ are the number of uninfected people within each arm's total population $N_p$; the number of possible states is $S$. Transmission within each arm is guided by parameters: $\kappa$, the average number of contacts within the arm's subpopulation in each round, and $r_{\textit{infect}}$, the probability of becoming infected after contact with an infected person. 

In this setting, there is a budget constraint over interventions. There are three available intervention actions ${a_0, a_1, a_2}$ that affect the transmission parameters: $a_0$ represents no action; $a_1$ represents messaging about physical distancing; $a_2$ represents the distribution of face masks. We discuss additional details in Appendix~\ref{sec:appendix_exp_details}.

{\bf ARMMAN:} Similar to the set up in \cite{biswas2021learn,killian2022restless}, we model the real world maternal health problem as a discrete state RMAB. We aim to encourage engagement with automated health information messaging. There are three possible states, presenting self-motivated, persuadable, and lost cause. The actions are binary. There are 6 uncertain parameters per arm, sampled from uncertainty intervals of
0.5 centered around the transition parameters that align with summary statistics given in \cite{biswas2021learn}.

{\bf Continuous State Modeling: } Continuous state restless bandits have important applications \cite{lefevre1981optimal,sinha2022robustness,dusonchet2003continuous}. By not explicitly having a switch in the model (switching between discrete and continuous state space), we enable greater model flexibility.
To demonstrate this, we consider both a Continuous Synthetic and a Continuous SIS modeling setting. We provide details of these settings in Appendix~\ref{sec:appendix_continuous_state_details}. 

We present {\bf additional details}, including hyperparameters and StateShaping illustration in Appendix~\ref{sec:appendix_exp_details}.

\subsection{PreFeRMAB Zero-Shot Learning}\label{sec:exp_discrete_state}
We first consider three challenging datasets in the discrete state space. After that, we present results on datasets with continuous state spaces with more complex reward functions and transition dynamics. 

{\bf Pretraining. } For each pretraining iteration, we sample from a binomial with mean $0.8$ to determine which arms will be opted-in given system capacity $N$. For new arms, we sample new transition dynamics to allow the model to see a wider range of arm features.


{\bf Evaluation. }\ We compare PreFeRMAB to \textit{Random Action} and \textit{No Action} baselines. In every table in this subsection, we present the \textit{reward per arm} averaged over 50 trials, on \textit{new, unseen arms} arm sampled from the testing distribution. 


\begin{table}[h!]
\scalebox{.77}{
    \begin{tabular}{@{}lccc@{}}
        \toprule
        \multicolumn{4}{c}{\textbf{System capacity $N=21$.  Budget $B=7$.}}  \\ 
        \# Unique training arms & 45 & 33 & 21 \\ \midrule
        No Action                           & $2.88_{\pm0.17}$& $2.88_{\pm0.17}$& $2.88_{\pm0.17}$\\
        Random Action                       & $3.25_{\pm0.22}$& $3.25_{\pm0.22}$& $3.25_{\pm0.22}$\\
        { PreFeRMAB (2/4 Feats. Masked)}    & $3.81_{\pm0.23}$& $3.79_{\pm0.22}$& $3.59_{\pm0.21}$\\
        { PreFeRMAB (1/4 Feats. Masked)}    & $3.92_{\pm0.24}$& $3.70_{\pm0.21}$& $3.58_{\pm0.20}$\\
        {\bf PreFeRMAB (0/4 Feats. Masked)} & $4.02_{\pm0.26}$& $3.80_{\pm0.22}$& $3.78_{\pm0.21}$\\
       \midrule
    \end{tabular}
    }
    \caption{Multi-arm generalization results on Synthetic (opt-in 100\%). With the same total amount of data, PreFeRMAB achieves stronger performance when pretrained on more unique arms, especially when input arm features are masked. } 
    \label{table:uniquesamples_vs_reward}
\end{table}

{\bf Multi-arm Generalization: } Table~\ref{table:uniquesamples_vs_reward} on Synthetic illustrates that PreFeRMAB, learning from \textit{multi-arm generalization}, achieves stronger performance when the number of unique arms (i.e. arms with unique features) seen during pretraining increases. Additionally, in practice arm features may be missing or not always reliable, such as in real-world ARMMAN data \cite{mate2022field}. Our results demonstrate that when features are masked, arms could learn from similar arms' experience.



\begin{table}[h!]
\scalebox{.75}{
    \begin{tabular}{@{}lccccc@{}}
        \toprule
        $\textbf{Wasserstein}$& \textbf{0.05}& \textbf{0.10}& \textbf{0.15}& \textbf{0.20}& \textbf{0.25}\\
        $\textbf{Distance}$   & & & & &      
        \\ \midrule
        \multicolumn{6}{c}{System capacity $N=48$.  Budget $B=16$.}                                                                                     \\
        \midrule
        No Action                                                   & $3.07_{\pm0.10}$& $2.89_{\pm0.08}$& $2.68_{\pm0.07}$& $2.49_{\pm0.09}$& $2.35_{\pm0.07}$\\
        Random Action                                                   & $3.49_{\pm0.09}$& $3.25_{\pm0.09}$& $2.99_{\pm0.16}$& $2.80_{\pm0.17}$& $2.57_{\pm0.17}$\\
        {\bf PreFeRMAB}                                                     & $4.50_{\pm0.09}$& $4.30_{\pm0.10}$& $3.81_{\pm0.17}$& $3.79_{\pm0.18}$& $3.46_{\pm0.12}$\\
        \midrule
        \multicolumn{6}{c}{System capacity $N=96$.  Budget $B=32$.}                                                                                     \\
        \midrule
        No Action                                                   & $3.09_{\pm0.08}$& $2.88_{\pm0.04}$& $2.74_{\pm0.05}$& $2.62_{\pm0.06}$& $2.49_{\pm0.06}$\\
        Random Action                                                   & $3.44_{\pm0.14}$& $3.23_{\pm0.09}$& $3.05_{\pm0.09}$& $2.90_{\pm0.10}$& $2.70_{\pm0.11}$\\
        {\bf PreFeRMAB}                                                     & $4.44_{\pm0.13}$& $4.26_{\pm0.13}$& $4.12_{\pm0.13}$& $3.97_{\pm0.16}$& $3.75_{\pm0.12}$\\
        \bottomrule
    \end{tabular}
    }
    \caption{Results on Synthetic (opt-in 100\%). For each system capacity, we pretrain a model and present zero-shot results under different amounts of distributional shift.  }
    \label{table:dist-shift}
\end{table}

\begin{table}[h!]
\scalebox{.75}{
    \begin{tabular}{@{}lccccc@{}}
        \toprule
        $\frac{\textbf{Number of arms}}{\textbf{System capacity}}$ & \textbf{80\%} & \textbf{85\%} & \textbf{90\%} & \textbf{95\%} & \textbf{100\%}\\ \midrule
        \multicolumn{6}{c}{Parameters $a_1^{eff}, a_1^{eff}$ are uniformly sampled from $[2,8]$.} \\
        \midrule
        No Action                                                   & $5.23_{\pm 0.17}$ & $5.27_{\pm 0.16}$ & $5.28_{\pm 0.16}$ & $5.26_{\pm 0.14}$ & $5.28_{\pm 0.13}$ \\
        Random Action                                                   & $6.94_{\pm0.15}$ & $7.00_{\pm0.16}$ & $7.03_{\pm0.15}$ & $6.97_{\pm0.14}$ & $6.99_{\pm0.12}$ \\
        {\bf PreFeRMAB}                                                     & $7.64_{\pm 0.27}$ & $7.75_{\pm0.25}$ & $7.96_{\pm0.18}$ & $7.80_{\pm0.16}$ & $7.82_{\pm0.11}$  \\
        \midrule
        \multicolumn{6}{c}{Parameters $a_1^{eff}, a_1^{eff}$ are uniformly sampled from $[3,9]$.}                                                                                     \\
        \midrule
        No Action                                                   
        & $5.29_{\pm0.16}$ & $5.30_{\pm0.17}$ & $5.29_{\pm0.15}$ & $5.26_{\pm0.14}$ & $5.28_{\pm0.13}$ \\
        Random Action                                                   
        & $7.21_{\pm0.15}$ & $7.28_{\pm0.18}$ & $7.26_{\pm0.15}$ & $7.22_{\pm0.13}$ & $7.22_{\pm0.12}$ \\
        {\bf PreFeRMAB}                                                     
        & $7.77_{\pm0.29}$ & $7.87_{\pm0.28}$ & $7.90_{\pm0.22}$ & $7.95_{\pm0.16}$ & $7.95_{\pm0.11}$ \\
        \bottomrule
    \end{tabular}
    }
    \caption{Results on SIS ($N=20,B=16,S=150$). We pretrain a model and present zero-shot results on various distributions. During training, $a_1^{eff}, a_1^{eff}$ are uniformly sampled from $[1,7]$. }
    \label{table:dist-shift-sis}
\end{table}

{\bf Discrete State Settings with Different Distributional Shifts: }
Results on Synthetic (Table~\ref{table:dist-shift}) shows PreFeRMAB consistently outperforms under varying amounts of distributional shift, measured in Wasserstein distance. Results on SIS  (Table~\ref{table:dist-shift-sis}) shows PreFeRMAB performs well in settings with large state space $S=150$ and multiple actions, under various testing distributions and opt-in rates. Results on ARMMAN  (Table~\ref{table:dist-shift-armman}) shows PreFeRMAB could handle more \textit{challenging} settings that mimics the scenario of a real-world non-profit organization using RMABs to allocate resources.

\begin{table}[h!]
\scalebox{.75}{
    \begin{tabular}{@{}lccccc@{}}
        \toprule
        $\frac{\textbf{Number of arms}}{\textbf{System capacity}}$ & \textbf{80\%} & \textbf{85\%} & \textbf{90\%} & \textbf{95\%} & \textbf{100\%}\\ \midrule
        \multicolumn{6}{c}{40\% motivated, 20\% persuadable, and 40\% lost cause.} \\
        \midrule
        No Action                                                   
        & $2.39_{\pm 0.30}$ & $2.32_{\pm 0.28}$ & $2.16_{\pm 0.25}$ & $2.26_{\pm 0.28}$ & $2.25_{\pm 0.30}$ \\
        Random Action                                                   
        & $3.04_{\pm 0.40}$ & $3.06_{\pm 0.38}$ & $3.00_{\pm 0.36}$ & $3.14_{\pm 0.36}$ & $3.24_{\pm 0.32}$  \\
        {\bf PreFeRMAB}                                                 
        & $5.47_{\pm 0.41}$ & $5.00_{\pm 0.37}$ & $4.95_{\pm 0.29}$ & $5.34_{\pm 0.27}$ & $5.03_{\pm 0.37}$   \\
        \midrule
        \multicolumn{6}{c}{40\% motivated, 40\% persuadable, and 20\% lost cause.}                                                                               \\
        \midrule
        No Action                                                   
        & $2.07_{\pm 0.29}$ & $2.19_{\pm 0.28}$ & $2.19_{\pm 0.30}$ & $2.05_{\pm 0.24}$ & $2.17_{\pm 0.28}$ \\
        Random Action                                                   
        & $3.04_{\pm 0.37}$ & $3.02_{\pm 0.33}$ & $2.99_{\pm 0.30}$ & $3.12_{\pm 0.31}$ & $3.15_{\pm 0.29}$  \\
        {\bf PreFeRMAB}                                                
        & $5.06_{\pm 0.36}$ & $4.81_{\pm 0.35}$ & $5.13_{\pm 0.34}$ & $5.01_{\pm 0.26}$ & $5.00_{\pm 0.28}$   \\
        \bottomrule
    \end{tabular}
    }
    \caption{Results on ARMMAN ($N=25,B=7,S=3$). We pretrain a model and present zero-shot results on various testing distributions. During training, the proportion of self-motivated, persuadable, and lost cause arms are 20\%, 20\%, annd 60\% respectively.}
    \label{table:dist-shift-armman}
\end{table}



\textbf{Continuous State Settings: } Our results (Table~\ref{table:state_shaping}) show that StateShaping is crucial in handling continuous states, where the reward function can be more challenging. We provide additional evaluations in Table~\ref{table:continuous_state_all}, showing PreFeRMAB outperforms in complex transition dynamics. More details are provided in Appendix~\ref{sec:appendix_exp_details}. 

\begin{table}[htb!]
\scalebox{.75}{
\centering
\begin{tabular}{@{}lcc@{}}
\toprule
\textbf{Reward Function} & \textbf{Scaled Linear} & \textbf{Exponential} \\
& \( r(s) = \min(2s, 1) \) & \( r(s) = \min(e^{s} - 1, 1) \) \\
\midrule
No Action & $1.03 \pm 0.32$ & $0.71 \pm 0.26$ \\
Random Action & $4.18 \pm 0.45$ & $3.40 \pm 0.43$ \\
{\bf PreFeRMAB} & $4.50 \pm 0.39$ & $3.80 \pm 0.35$ \\
\textbf{{\bf PreFeRMAB}, With Shaping} & \textbf{5.28 $\pm$ 0.55} & \textbf{4.38 $\pm$ 0.55} \\
\bottomrule
\end{tabular}
}
\caption{Results for \textit{Continuous Synthetic} domain (N=21,B=7.0) \textit{with challenging rewards} $r(s)$.}
\label{table:state_shaping}
\end{table}

\begin{table}[htb!]
\scalebox{.75}{
\begin{tabular}{@{}lccccc@{}}
\toprule
$\frac{\textbf{Number of arms}}{\textbf{System capacity}}$ & \textbf{80\%} & \textbf{85\%} & \textbf{90\%} & \textbf{95\%} & \textbf{100\%}\\ \midrule
\multicolumn{6}{c}{Continuous Synthetic (N=21, B=7.0, S=2)} \\
\midrule
  No Action       & $0.70_{\pm 0.25}$ & $0.71_{\pm 0.25}$ & $0.70_{\pm 0.23}$ & $0.70_{\pm 0.23}$ & $0.66_{\pm 0.21}$ \\
  Random Action & $3.44_{\pm 0.48}$ & $3.43_{\pm 0.45}$ & $3.45_{\pm 0.41}$ & $3.37_{\pm 0.41}$ & $3.20_{\pm 0.38}$ \\
  {\bf PreFeRMAB}          & $3.94_{\pm 0.31}$ & $3.76_{\pm 0.33}$ & $4.01_{\pm 0.31}$ & $4.02_{\pm 0.29}$ & $3.67_{\pm 0.27}$ \\
  \midrule
\multicolumn{6}{c}{Continuous SIS Model (N=20, B=16)} \\
\midrule
  No Action       & $5.64_{\pm 0.25}$ & $5.68_{\pm 0.19}$ & $5.57_{\pm 0.21}$ & $5.48_{\pm 0.18}$ & $5.62_{\pm 0.17}$ \\
  Random Action & $7.11_{\pm 0.22}$ & $7.31_{\pm 0.22}$ & $7.23_{\pm 0.22}$ & $7.24_{\pm 0.21}$ & $7.18_{\pm 0.17}$ \\
    {\bf PreFeRMAB}          & $7.91_{\pm 0.17}$ & $8.08_{\pm 0.11}$ & $7.95_{\pm 0.14}$ & $7.98_{\pm 0.13}$ & $7.82_{\pm 0.12}$ \\
\bottomrule
\end{tabular}
}
\caption{Results on continuous states. For each problem instance, we pretrain a model.}
\label{table:continuous_state_all}
\end{table}

\begin{table}[htb!]
\scalebox{.75}{
\begin{tabular}{@{}lccccc@{}}
\toprule
$\frac{\textbf{Number of arms}}{\textbf{System capacity}}$ & \textbf{80\%} & \textbf{85\%} & \textbf{90\%} & \textbf{95\%} & \textbf{100\%}\\ \midrule
\multicolumn{6}{c}{Synthetic with $N=96,B=32,S=2$.} \\
\midrule
  No Action                  & $3.22_{\pm 0.12}$ & $3.24_{\pm 0.12}$ & $3.19_{\pm 0.11}$ & $3.18_{\pm 0.11}$ & $3.18_{\pm 0.11}$ \\
  Random Action              & $3.62_{\pm 0.13}$ & $3.66_{\pm 0.13}$ & $3.58_{\pm 0.13}$ & $3.60_{\pm 0.12}$ & $3.60_{\pm 0.12}$ \\
  {\bf PreFeRMAB}                   & $4.63_{\pm 0.12}$ & $4.71_{\pm 0.12}$ & $4.53_{\pm 0.13}$ & $4.47_{\pm 0.12}$ & $4.61_{\pm 0.10}$ \\
  DDLPO (topline)                & \textit{n/a} & \textit{n/a} & \textit{n/a} & \textit{n/a} & $4.58_{\pm 0.13}$ \\
  \midrule
\multicolumn{6}{c}{SIS with $N=20,B=16,S=150$.} \\
\midrule
  No Action & $5.33_{\pm 0.16}$ & $5.30_{\pm 0.15}$ & $5.31_{\pm 0.14}$ & $5.29_{\pm 0.13}$ & $5.28_{\pm 0.13}$\\
  Random Action & $7.03_{\pm 0.17}$ & $7.13_{\pm 0.16}$ & $7.02_{\pm 0.14}$ & $7.11_{\pm 0.13}$ & $7.06_{\pm 0.13}$\\
{\bf PreFeRMAB} & $8.35_{\pm 0.12}$ & $8.38_{\pm 0.11}$ & $8.26_{\pm 0.11}$ & $8.10_{\pm 0.11}$ & $8.00_{\pm 0.10}$\\
DDLPO (topline)                & \textit{n/a} & \textit{n/a} & \textit{n/a} & \textit{n/a} & $8.09_{\pm 0.11}$ \\
    \midrule
\multicolumn{6}{c}{ARMMAN with $N=25,B=7,S=3$.} \\
\midrule
  No Action & $2.12_{\pm 0.26}$ & $2.30_{\pm 0.29}$ & $2.29_{\pm 0.27}$ & $2.19_{\pm 0.23}$ & $2.26_{\pm 0.25}$\\
  Random Action & $2.86_{\pm 0.32}$ & $3.27_{\pm 0.40}$ & $3.01_{\pm 0.30}$ & $3.09_{\pm 0.35}$ & $2.96_{\pm 0.31}$\\
  {\bf PreFeRMAB} & $5.06_{\pm 0.34}$ & $5.26_{\pm 0.33}$ & $4.68_{\pm 0.33}$ & $4.75_{\pm 0.35}$ & $4.61_{\pm 0.27}$\\
  DDLPO (topline)                & \textit{n/a} & \textit{n/a} & \textit{n/a} & \textit{n/a} & $4.68 _{\pm 0.09}$ \\
\bottomrule
\end{tabular}
}
\caption{Comparison of PreFeRMAB \text{zero-shot} performance on unseen arms against that of DDLPO trained and tested on the same set of arms. For each problem instance, we pretrain a model.}
\label{table:discrete_all_with_topline}
\end{table}

\textbf{Comparison with an Additional Baseline: } DDLPO \cite{killian2022restless} could not handle distributional shifts or various opt-in rates, the more challenging settings that PreFeRMAB is designed for. Nevertheless, we provide comparisons with DDLPO in settings with no distributional shift (Table~\ref{table:discrete_all_with_topline}, see also Appendix~\ref{sec:appendix_different_n_b_s}). Notably, PreFeRMAB zero-shot performance on \textit{unseen} arms is near that of DDLPO, which is trained and tested on the same set of arms.

\subsection{PreFeRMAB Fast Fine-Tuning}
\label{sec:prefermab-fast-ft}
Having shown the zero-shot results of PreFeRMAB, we now demonstrate finetuning capabilities of the pretrained model. In Figure~\ref{fig:fine_tuning}, we compare the number of samples required to \textit{train DDLPO from scratch} vs. the number of samples for \textit{fine-tuning PreFeRMAB} starting from a pre-trained model (additional results in Appendix \ref{sec:appendix_fine_tuning}). Results suggests the cost of pretraining can be amortized over different downstream instances. A non-profit organization using RMAB models may have new beneficiaries opting in every week, and training a new model from scratch every week can be 3-20 times more expensive than  fine-tuning our pretrained model. 

\begin{figure}[!ht]
  \centering
  \includegraphics[width=0.5\textwidth]{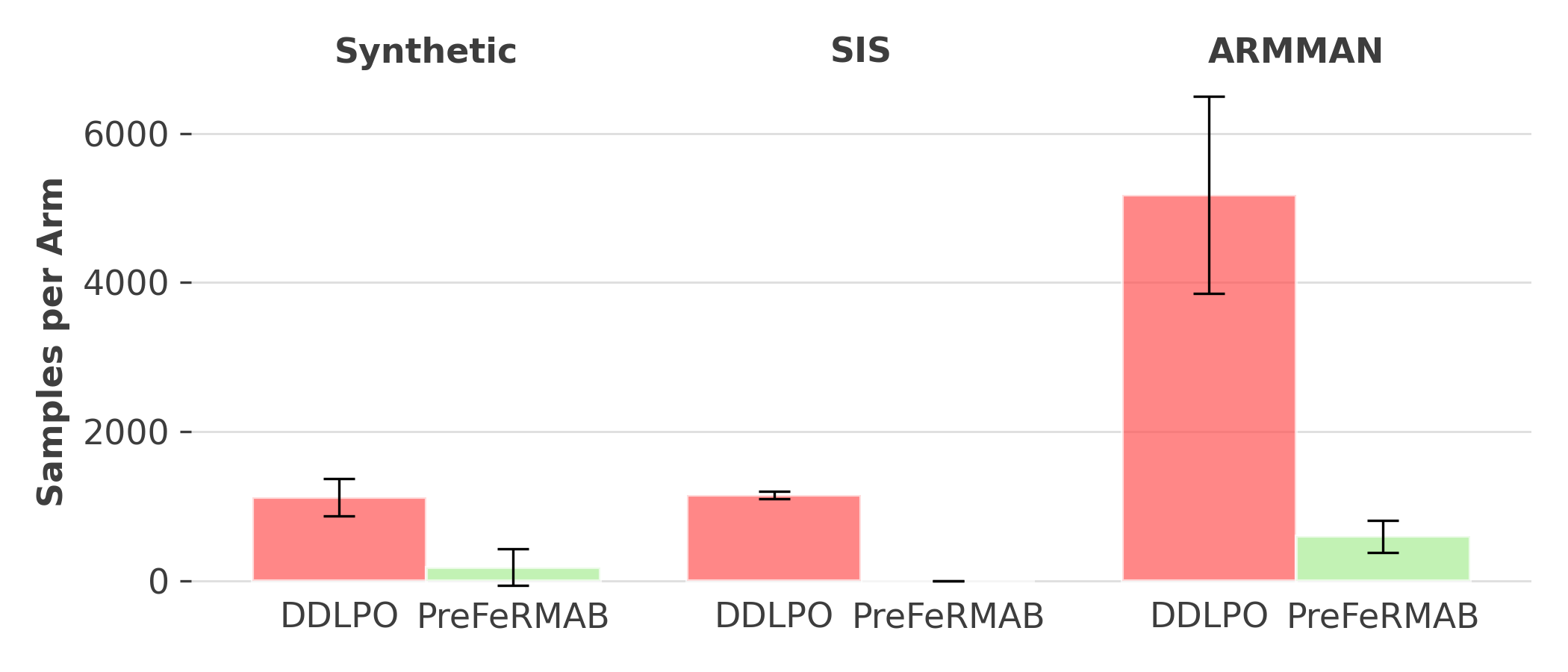}
  \caption{Comparison of samples per arm required by DDLPO and PreFeRMAB (fine-tuning using a pretrained model) to achieve maximum DDLPO reward across different environments. PreFeRMAB achieves the maximum topline reward with significantly fewer samples than DDLPO. Averages across training seeds are reported as interquartile means.}
  \label{fig:fine_tuning}
\end{figure}

\section{Conclusion}

Our pretrained model (PreFeRMAB) leverages multi-arm generalization, a novel update rule for a crucial $\lambda$-network, and a StateShaping module for challenging reward functions. PreFeRMAB demonstrates general zero-shot ability on unseen arms, and can be fine-tuned on specific instances in a more sample-efficient way than training from scratch.


\newpage
\section*{Ethical Statement} The presented methods do not carry direct negative societal implications. However, training reinforcement learning models should be done responsibly, especially given the safety concerns associated with agents engaging in extreme, unsafe, or uninformed exploration strategies. While the domains we considered such as ARMMAN do not have these concerns, the approach may be extended to extreme environments; in these cases, ensuring a robust approach to training reinforcement models is critical. 

\section*{Acknowledgments}
The work is supported by Harvard HDSI funding.

\bibliographystyle{named}
\bibliography{ijcai24}

\appendix
\newpage 

\section{Additional Experimental Details}
\label{sec:appendix_exp_details}

\begin{figure*}[!h]
    \centering
    \includegraphics[width=\textwidth]{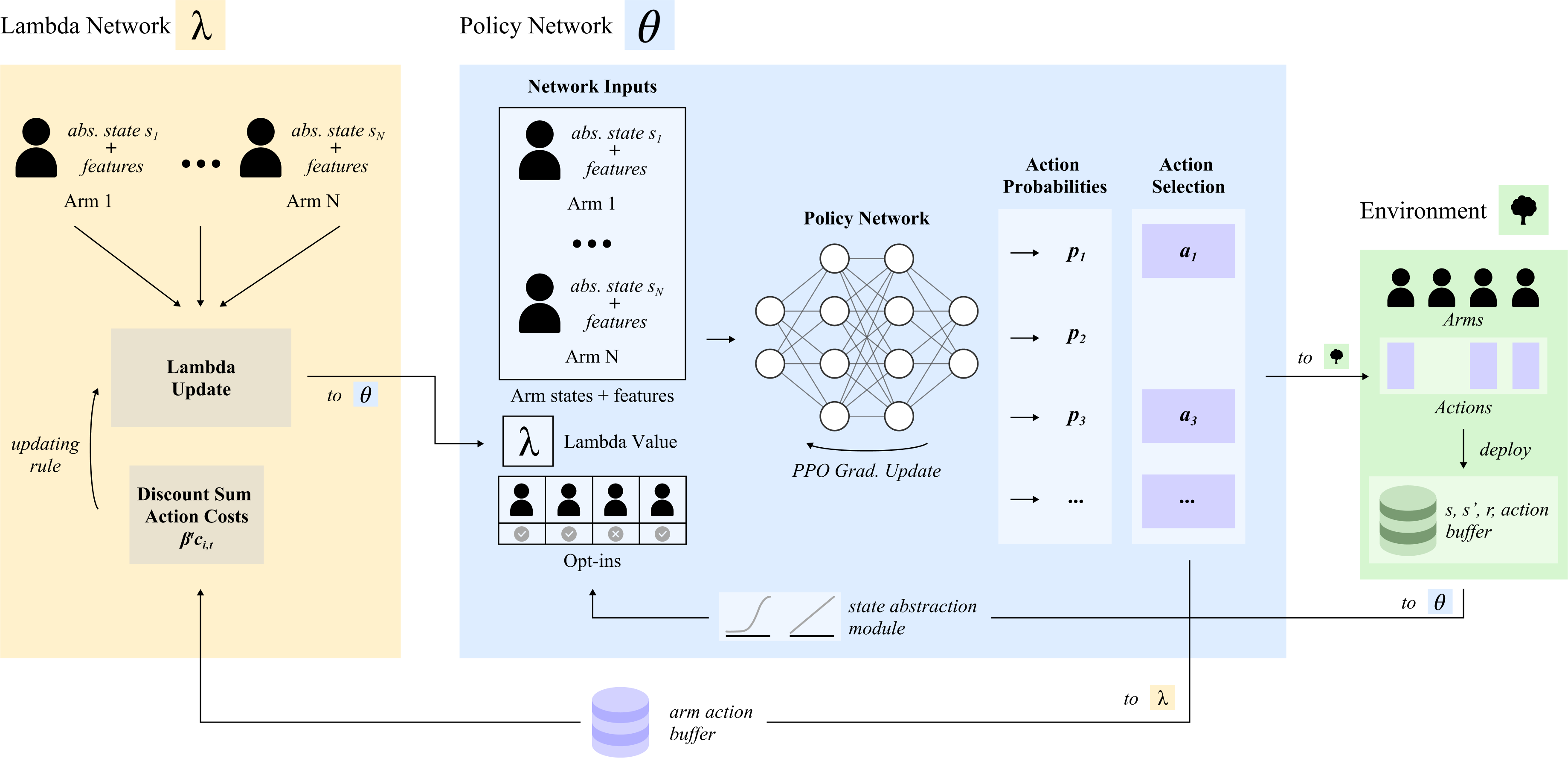}
    \caption{Overview of the PreFeRMAB training procedure. A trained model consists of a policy network, a critic network, a $\lambda$-network, and a StateShaping module. Arm states $s_i$, features $z_i$, and opt-in decisions $\xi$ are passed through the policy network with an action-charge $\lambda$. The policy network independently predicts action probabilities for each arm, which are then greedily selected until the specified budget is reached. These selected actions are used with arm state, feature, and opt-in information to update the $\lambda$-network. Updated arm states $s'$ and rewards $r$ from the environment are then added to the buffer, and passed through the state abstraction module before being fed back through the policy network.}
    \label{fig:overview}
\end{figure*}

\subsection{Hyperparameters}
In Table~\ref{table:hyperparameter}, we present hyperparameters used, with exceptions (1) for Continuous Synthetic, we use lambda scheduler discount rate = 0.95 (2) for Continuous SIS, we use training opt-in rate = 0.8.

In our experiments, all neural networks have 2 hidden layers each with 16 neurons and tanh activation. The output layer has identity activation and its size is determined by the number of actions (3 for SIS and Continuous SIS, and 2 for other environments). 

The $lambda$-network training is similar to that in Killian \textit{et al.}\shortcite{killian2022restless}. After every n\_subepochs, we update the $\lambda$-network and encourage the actor network to explore new parts of the state space immediately after the $lambda$-update (this exploration is controlled by the temperature parameter that weights the entropy term in the actor loss functions). 

Different from Killian \textit{et al.}\shortcite{killian2022restless}, we use a $\lambda$-network learning rate scheduler, which we found improves the performance and stability of the model.

\begin{table}[htb!]
    \caption{Hyperparameter values.}
\scalebox{.80}{
    \begin{tabular}{@{}lc@{}}
        \toprule
        hyperparameter & value \\ 
        \midrule
        training opt-in rate & 0.8\\
        agent clip ratio  & 2.0e+00\\
        lambda freeze epochs & 2.0e+01\\
        start entropy coeff & 5.0e-01\\
        end entropy coeff & 0.0e+00\\
        actor learning rate & 2.0e-03\\
        critic learning rate & 2.0e-03\\
        lambda initial learning rate & 2.0e-03\\
        lambda scheduler discount rate& 0.99\\
        trains per epoch & 2.0e+01\\
        n\_subepochs & 4.0e+00\\
        \bottomrule
    \end{tabular}
    }
    \label{table:hyperparameter}
\end{table}

\subsection{SIS Modeling (Discrete) Experimental Details}

Recall that each arm $p$ represents a subpopulation in distinct geographic regions. The state of each arm $s$ is the number of uninfected people within the arm's total population $N_p$. Transmission within each arm is guided by parameters: $\kappa$, the average number of contacts within the arm's subpopulation in each round, and $r_{\textit{infect}}$, the probability of becoming infected after contact with an infected person. The probability that a single uninfected person gets infected is then: 
$$
q = 1 - e^{-\kappa \cdot \frac{S-s}{S} \cdot r_{infect} },
$$
where $S$ is the number of possible states, and $s\in[S]$ is the current state. Note $\frac{S-s}{S}$ is the percentage of people who are currently infected. The number of infected people in the next timestep follows a binomial distribution $B(S,q)$.

Recall that there are three available intervention actions ${a_0, a_1, a_2}$ that affect the transmission parameters: $a_0$ represents no action; $a_1$ represents messaging about physical distancing, dividing $\kappa$ by $a_1^{\textit{eff}}$; $a_2$ represents the distribution of face masks, dividing $r_{\textit{infect}}$ by $a_2^{\textit{eff}}$. The actions costs are $c=\{0,1,2\}$. Following the implementation in \cite{killian2022restless}, these parameters are sampled within ranges: 
\[
\kappa \in [1, 10], \ r_{\textit{infect}} \in [0.5, 0.99], \ a_1^{\textit{eff}} \in [1, 10], a_2^{\textit{eff}} \in [1, 10]
\]

\subsection{Continuous States Experimental Details}
\label{sec:appendix_continuous_state_details}
We consider a synthetic dataset with continuous states and binary actions. For the current state $s_{i}$ of arm $i$, and action $a$, the next state $s_i'$ is represented by the transition dynamic: 
\[
s_{i}' = \begin{cases} 
\text{clip}\left(s_i + \mathcal{N}(\mu_{i0}, \sigma_{i0}), 0, 1\right) & \text{if } a = 0 \\
\text{clip}\left(s_i + \mathcal{N}(\mu_{i1}, \sigma_{i1}), 0, 1\right) & \text{if } a = 1 \\
\end{cases}
\]

Where the transition dynamics are sampled uniformly from the intervals ($\sigma_{i0}=\sigma_{i1} = 0.2$ is fixed): 
\[
\mu_{i0} \in [-0.5, -0.1], \quad \mu_{i1} \in [0.1, 0.5].
\]

We also consider continuous state experiments in real-world settings. In the discrete state SIS Epidemic Model described above, each arm represents a subpopulation, and the state of that arm represents the number of uninfected people within the subpopulation. In real-world public health settings such as COVID-19 control, interventions like quarantine and mask mandates may be imposed on subpopulations of very large sizes such as an entire city. The SIS model is expected to scale to a continuum limit as the population size increases to infinity. Thus, a SIS model with population 1 million would behave roughly similar to that with population 1 billion in terms of the proportions. This notion is inherently captured by continuous models but not by those dealing with absolute numbers. 

Following Killian \textit{et al.}\shortcite{killian2022restless}, within an arm, any uninfected person will get infected with the same probability. Thus, the number of uninfected people in the next timestep follows a binomial distribution. It is well-known that a normal distribution $\mathcal{N}(\mu,\sigma^2)$ well approximates a binomial distribution $B(n,p)$, with the choice $\mu=np$ and  $\sigma^2=np(1-p)$, when $n$ is sufficiently large. 

\subsection{Distributional Shift Details}
\label{sec:appendix_dist_shift_details}

In real-world resources allocation problems, we may observe distribution shifts in arms, i.e., arms in testing are sampled from a distribution slightly different from that in training. In public health settings, a non-profit organization solving RMAB problems to allocate resources to beneficiaries may observe that beneficiaries' behavior or feature information change over time \cite{wang2023scalable,killian2022restless}. Additionally, a non-profit organization may have new beneficiaries joining who are in a different subpopulation. In \autoref{table:dist-shift}, we provide ablation results illustrating that {\bf PreFeRMAB is robust to distribution shift in arms}. We measure the shift in distribution using Wasserstein distance. The results demonstrate that even on arm samples from distributions that significantly deviate from that seen in training, PreFeRMAB still achieves strong performance and outperforms baselines. 

For each arm, the associated Markov Decision Process (MDP) has only two discrete states, the transition dynamics \( p(s' | s, a) \), representing the probability of transitioning to state \( s' \) from state \( s \) given action \( a \), can be described by four Bernoulli random variables, one for each combination of state and action. By introducing a uniform distribution shift, we can modify the transition probabilities of the Bernoulli random variable associated with each state-action pair by adding a constant \( \delta \) to the parameter of each Bernoulli distribution. Consequently, this results in a consistent shift in the transition probabilities across all states and actions. 

Furthermore, this delta is exactly the Wasserstein distance between the two distributions, which we show here. Suppose we have two discrete probability distributions, \( P \) and \( Q \), with their respective probabilities associated with the outcomes \( x_i \) and \( y_j \). The The 1-Wasserstein distance, also known as the Earth Mover's Distance \( W(P, Q) \), can be calculated by solving the following optimization problem: 
\[ W(P, Q) = \inf_{\gamma \in \Gamma(P, Q)} \sum_{i,j} |x_i - y_j| \gamma(x_i, y_j), \] 

where \( \Gamma(P, Q) \) is the set of all joint distributions \( \gamma \) with marginals \( P \) and \( Q \), and \( \gamma(x_i, y_j) \) represents the amount of "mass" moved from \( x_i \) to \( y_j \). 

Note \( W(P, Q) \) between two Bernoulli distributions with parameters \( b_1 \) and \( b_2 \) can be succinctly determined as \( |b_2 - b_1| \). This is because each Bernoulli distribution has only two potential outcomes, 0 and 1, and so moving mass from one outcome to another across these distributions involves a shift of probability mass \( |b_2 - b_1| \) across the one-unit distance between the two points. Therefore, without loss of generality assuming \( b_2 \geq b_1 \), the Wasserstein distance simplifies to the non-negative difference \( b_2 - b_1 \).

\subsection{Fast Fine-Tuning}
\label{sec:appendix_fine_tuning}

\begin{figure}[!ht]
  \centering
  \includegraphics[width=0.4\textwidth]{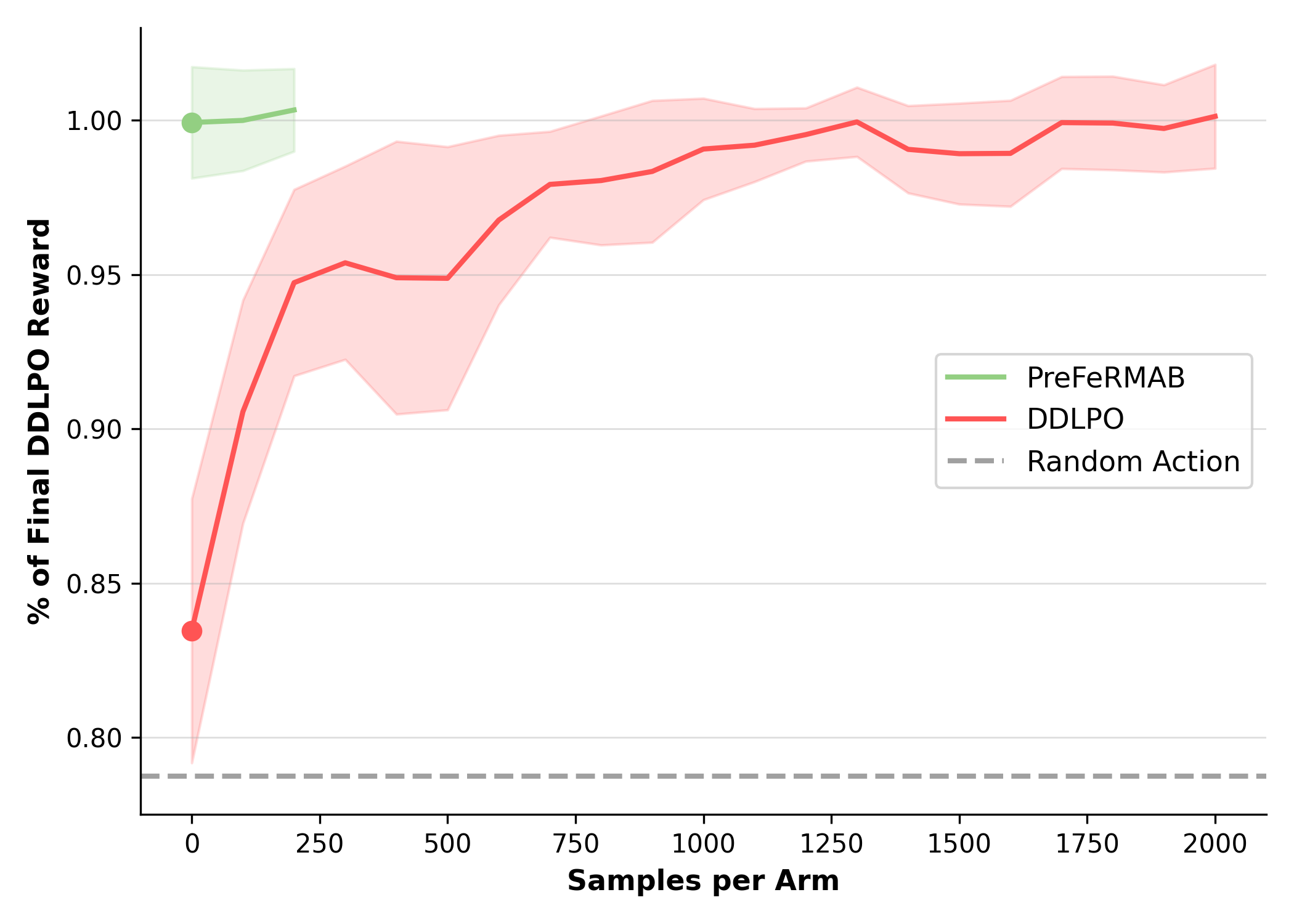}
  \caption{Comparison of the percentage of the final DDLPO (Killian \textit{et al.} [2022] topline) reward achieved by the number of samples per arm. In DDLPO, samples are used for training from scratch; in PreFeRMAB, samples are used to fine-tune a pretrained PreFeRMAB model. Results indicate that PreFeRMAB, from zero-shot results, achieves near-optimal performance, and requires a small fraction of the required DDLPO samples to achieve final DDLPO performance. }
  \label{fig:fine_tuning_graph}
\end{figure}

In \autoref{sec:prefermab-fast-ft}, we demonstrate that, in addition to strong zero-shot performance, PreFeRMAB may also be used as a \textit{pretrained model} for fast fine-tuning in specific domains. In particular, we demonstrate that we may start from a pre-trained PreFeRMAB model, and train on additional samples for a fixed environment (with fixed arm transition dynamics). We showed that using this pre-trained model can help achieve topline DDLPO performance in significantly fewer fine-tuning samples than required by DDLPO to train from scratch. In \autoref{fig:fine_tuning_graph}, we further visualize these results in training curves comparing DDLPO and PreFeRMAB. These training curves, which plot the number of samples per arm against the achieved percentage of final DDLPO reward, are shown for the discrete-state synthetic environment setting for N=21, B=7.0 

The results in \autoref{fig:fine_tuning_graph} demonstrate that PreFeRMAB shows both 1) strong zero-shot performance, achieving near-topline reward with \textit{no fine-tuning samples required}, as well as 2) a significant reduction in the number of samples required to achieve final DDLPO performance. In particular, we note that DDLPO, before training, achieves a reward only marginally higher than the average Random Action reward. Alternatively, PreFeRMAB begins, in a zero-shot setting, with a much higher initial reward value. We also observe that PreFeRMAB requires significantly fewer samples per arm to achieve the final DDLPO reward. This is particularly critical in high-stakes, real-world settings where continually sampling arms from the environment may be prohibitively expensive, especially for low-resource NGOs.

\subsection{StateShaping}
\label{sec:appendix_state_shaping}

Figure~\ref{fig:state_shaping} provides a simple example, illustrating how we adapt states through the state abstraction procedure. In this particular example, the reward is an increasing function of the state, and the reward plateaus at state 0.5, i.e. $s\in[0.5, 1]$ achieve the same reward. We map all raw observations in the range $[0.5, 1]$ to abstract state 1. We note that this process is automated, using data collected on arm states and reward from prior (historical) samples to 1) estimate the reward of a current arm, and 2) use this reward to normalize the arm state. We demonstrate how these states are mapped to normalized values in Table \ref{fig:state_shaping}. 

In Table \ref{table:state_shaping}, we show results in two settings after 30 epochs of training, evaluating on a separate set of test arms for zero-shot evaluation. Continuous transition dynamics are used directly as input features for training and evaluation. The results illustrate that state abstraction can help achieve additional performance gains for various challenging reward functions. Drawing from prior literature on state abstraction, this modular component of PreFeRMAB may also serve as a placeholder for future automated state abstraction procedures to improve generalizability and robustness of PreFeRMAB across domains with challenging reward functions. 


\begin{figure}[h]
    \centering
     \begin{tabular}{cccc}
        \includegraphics[width=0.45\textwidth]{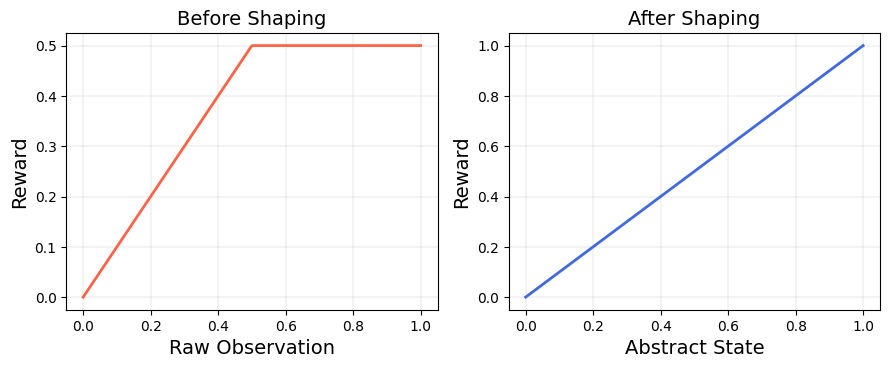}
     \end{tabular}
\caption{Illustration for StateShaping.}
\label{fig:state_shaping}
\end{figure}


\section{Ablation Studies}
\label{sec:appendix_ablation}
In this section, we provide ablation results over (1) a wider range of opt-in rates than presented in the main paper (Table~\ref{table:discrete_state_synthetic}) (2) different feature mappings, including linear and non-linear feature transformations of the original transition probabilities (3) DDLPO topline (Killian \textit{et al.}~\shortcite{killian2022restless}) with and without transition probability features as inputs (4) results in more problem settings. The ablation results showcase that PreFeRMAB consistently achieve strong performance and having access to feature information does not provide PreFeRMAB an unfair advantage over DDLPO.

\subsection{Opt-in Rates}
\begin{table}[h!]
\scalebox{.65}{
    \begin{tabular}{@{}lccccc@{}}
        \toprule
        $\frac{\textbf{Number of arms}}{\textbf{System capacity}}$ & \textbf{30\%}& \textbf{40\%}& \textbf{50\%}& \textbf{60\%}& \textbf{70\%}\\ \midrule
        \multicolumn{6}{c}{System capacity $N=21$. Budget $B=7$.}                                                                                       \\
        \midrule
        No Action                                                   & $3.09\pm0.31$& $3.10\pm0.32$& $3.12\pm0.29$& $3.14\pm0.25$& $3.16\pm0.22$\\
        Random Action                                                   & $3.57\pm0.48$& $3.46\pm0.48$& $3.55\pm0.35$& $3.55\pm0.34$& $3.57\pm0.33$\\
        {\bf PreFeRMAB}                                                     & $3.78\pm0.72$& $3.75\pm0.70$& $4.16\pm0.57$& $4.52\pm0.54$& $4.45\pm0.42$\\
        \midrule
        \multicolumn{6}{c}{System capacity $N=48$.  Budget $B=16$.}                                                                                     \\
        \midrule
        No Action                                                   & $3.19\pm0.27$& $3.15\pm0.23$& $3.13\pm0.12$& $3.17\pm0.12$& $3.17\pm0.14$\\
        Random Action                                                   & $3.44\pm0.26$& $3.43\pm0.23$& $3.46\pm0.20$& $3.47\pm0.17$& $3.44\pm0.17$\\
        {\bf PreFeRMAB}                                                     & $3.90\pm0.50$& $3.64\pm0.30$& $3.87\pm0.27$& $3.85\pm0.25$& $4.06\pm0.31$\\
        \midrule
        \multicolumn{6}{c}{System capacity $N=96$.  Budget $B=32$.}                                                                                     \\
        \midrule
        No Action                                                   & $3.21\pm0.17$& $3.17\pm0.20$& $3.17\pm0.15$& $3.18\pm0.14$& $3.17\pm0.13$\\
        Random Action                                                   & $3.54\pm0.22$& $3.55\pm0.23$& $3.58\pm0.18$& $3.55\pm0.18$& $3.56\pm0.13$\\
        {\bf PreFeRMAB}                                                     & $3.93\pm0.33$& $3.72\pm0.23$& $3.79\pm0.18$& $4.02\pm0.21$& $4.16\pm0.25$\\
        \bottomrule
    \end{tabular}
    }
    \caption{Robustness to different opt-in rates with identity mapping. Evaluation follows \autoref{table:discrete_state_synthetic_nonlin_transform}: we run for 50 trials with 2 total number of states for each arm, and pretrain a model for each system capacity $N$ and test generalization on different opt-in rates. }
    \label{table:robust-opt-in-rates}
\end{table}

\begin{table}[h!]
\scalebox{.65}{
    \begin{tabular}{@{}lccccc@{}}
        \toprule
        $\frac{\textbf{Number of arms}}{\textbf{System capacity}}$ & \textbf{30\%}& \textbf{40\%}& \textbf{50\%}& \textbf{60\%}& \textbf{70\%}\\ \midrule
        \multicolumn{6}{c}{System capacity $N=21$. Budget $B=7$.}                                                                                       \\
        \midrule
        No Action                                                   & $3.13\pm0.36$& $3.19\pm0.34$& $3.17\pm0.27$& $3.17\pm0.23$& $3.15\pm0.24$\\
        Random Action                                                   & $3.60\pm0.53$& $3.62\pm0.40$& $3.59\pm0.34$& $3.58\pm0.36$& $3.57\pm0.32$\\
        {\bf PreFeRMAB}                                                     & $3.79\pm0.49$& $3.76\pm0.49$& $3.79\pm0.39$& $3.89\pm0.39$& $3.86\pm0.46$\\
        \midrule
        \multicolumn{6}{c}{System capacity $N=48$.  Budget $B=16$.}                                                                                     \\
        \midrule
        No Action                                                   & $3.15\pm0.23$& $3.17\pm0.21$& $3.18\pm0.16$& $3.18\pm0.13$& $3.19\pm0.14$\\
        Random Action                                                   & $3.41\pm0.29$& $3.52\pm0.27$& $3.50\pm0.20$& $3.49\pm0.22$& $3.48\pm0.17$\\
        {\bf PreFeRMAB}                                                     & $3.61\pm0.51$& $4.08\pm0.34$& $4.45\pm0.30$& $4.44\pm0.30$& $4.44\pm0.29$\\
        \midrule
        \multicolumn{6}{c}{System capacity $N=96$.  Budget $B=32$.}                                                                                     \\
        \midrule
        No Action                                                   & $3.18\pm0.18$& $3.17\pm0.13$& $3.18\pm0.14$& $3.19\pm0.14$& $3.16\pm0.15$\\
        Random Action                                                   & $3.54\pm0.22$& $3.57\pm0.22$& $3.55\pm0.14$& $3.58\pm0.15$& $3.56\pm0.15$\\
        {\bf PreFeRMAB}                                                     & $3.74\pm0.21$& $3.68\pm0.23$& $3.76\pm0.21$& $3.98\pm0.19$& $4.06\pm0.22$\\
        \bottomrule
    \end{tabular}
    }
    \caption{Robustness to different opt-in rates with linear-mappings. Evaluation follows \autoref{table:discrete_state_synthetic_nonlin_transform}: we run for 50 trials with 2 total number of states for each arm, and pretrain a model for each system capacity $N$ and test generalization on different opt-in rates.  }
    \label{table:robust-opt-in-rates-linear}
\end{table}

Throughout the main paper, we provide results for evaluation opt-in rates in the range 80\%-100\%. In \autoref{table:robust-opt-in-rates} and \autoref{table:robust-opt-in-rates-linear}, we provide ablation results for {\bf opt-in rates in a wider range} of 30\%-70\%. During the training phase, we maintain an expected opt-in rate of 80\%, which may generally range from 70\%-90\% every training iteration. Given this training configuration, we demonstrate strong results in the main paper for evaluating on a similar range of test-time opt-ins from 80\% to 100\%. However, we also further demonstrate in \autoref{table:robust-opt-in-rates} and \autoref{table:robust-opt-in-rates-linear} that our pretrained PreFeRMAB model, despite a training opt-in rate around 80\% in expectation, achieves strong results on \textit{testing} opt-in rates from a substantially different range. These results highlight PreFeRMAB's flexibility and ability to generalize to unseen opt-in rates, which may be critical in real-world applications where arms frequently exit and re-enter the environment.

\subsection{Feature Mapping}
\begin{table}[h!]
\scalebox{.65}{
    \begin{tabular}{@{}lccccc@{}}
        \toprule
        $\frac{\textbf{Number of arms}}{\textbf{System capacity}}$ & \textbf{80\%}     & \textbf{85\%} & \textbf{90\%} & \textbf{95\%} & \textbf{100\%} \\ \midrule
        \multicolumn{6}{c}{System capacity $N=21$. Budget $B=7$.}                                                                                       \\
        \midrule
        No Action                                                   & $3.46\pm0.20$     & $3.39\pm0.19$ & $3.40\pm0.17$ & $3.40\pm0.18$ & $3.22\pm0.16$ \\
        Random Action                                                   & $3.80\pm0.31$      & $3.76\pm0.30$  & $3.79\pm0.29$ & $3.76\pm0.31$ & $3.58\pm0.27$ \\
        {\bf PreFeRMAB}                                                     & $4.57\pm0.29$       & $4.70\pm0.33$ & $4.70\pm0.29$  & $4.64\pm0.29$ & $4.37\pm0.25$ \\
        \midrule
        \multicolumn{6}{c}{System capacity $N=48$.  Budget $B=16$.}       
\\
        \midrule
        No Action                                                   & $3.22\pm0.13$      & $3.28\pm0.13$  & $3.22\pm0.12$ & $3.29\pm0.12$ & $3.21\pm0.11$ \\
        Random Action                                                   & $3.56\pm0.19$     & $3.65\pm0.18$ & $3.56\pm0.18$ & $3.65\pm0.17$ & $3.57\pm0.17$ \\
        {\bf PreFeRMAB}                                                     & $3.94\pm0.23$     & $4.00\pm0.23$ & $3.86\pm0.20$ & $3.90\pm0.16$ & $3.73\pm0.16$ \\
        \midrule
        \multicolumn{6}{c}{System capacity $N=96$.  Budget $B=32$.}                                                                                     \\
        \midrule
        No Action                                                   & $3.24\pm0.10$     & $3.24\pm0.10$ & $3.21\pm0.10$ & $3.21\pm0.10$ & $3.20\pm0.10$ \\
        Random Action                                                   & $3.61\pm0.14$     & $3.62\pm0.15$ & $3.57\pm0.15$ & $3.57\pm0.13$ & $3.57\pm0.14$ \\
        {\bf PreFeRMAB}                                                     & $4.35\pm0.16$     & $4.36\pm0.15$ & $4.29\pm0.14$ & $4.23\pm0.13$ & $4.20\pm0.12$ \\
        \bottomrule
    \end{tabular}
    }
    \caption{Results on non-linearly transformed synthetic discrete states. We present final reward divided by the number of arms, averaged over 50 trials with each trial consisting on 10 rounds, for a total of 500 evaluations. The number of states $S=2$. For each system capacity $N$, we pretrain a model
    }
    \label{table:discrete_state_synthetic_nonlin_transform}
\end{table}

In our main paper, we use linear feature mapping, projecting true transition probabilities to features with randomly generated projection matrices. This can be represented by \( \mathbf{y} = \mathbf{Ax} \), where \( \mathbf{y} \) are the output features, \( \mathbf{A} \) is the transformation matrix, and \( \mathbf{x} \) denotes the ground truth arm transition probabilities. To demonstrate the robustness of our approach to various types of input features, we also consider {\bf more challenging, non-linear feature mappings}, which may introduce higher representational complexity as compared to linear feature mappings. For these ablation results, we use a sigmoidal transformation, which can be expressed as \( \mathbf{y} = \frac{1}{1 + \exp(-\mathbf{Ax})} \). We demonstrate the results using these non-linear feature mappings in \autoref{table:discrete_state_synthetic_nonlin_transform}. These results indicate that PreFeRMAB consistently outperforms baselines under various forms of feature mappings, and is robust to both linear and non-linear input features.

\subsection{DDLPO Topline with Features}
\begin{table}[h!]
\scalebox{.87}{
    \begin{tabular}{@{}lccc@{}}
        \toprule
        \textbf{Synthetic Experiment} & \textbf{N=21,B=7.0}& \textbf{N=48,B=16.0} & \textbf{N=96,B=32.0} \\
        \midrule
        {\bf DDLPO, w/o Features} & $4.63\pm0.21$& $4.60\pm0.18$ & $4.35\pm0.11$ \\
        {\bf DDLPO, w/ Features} & $4.63\pm0.23$& $4.59\pm0.17$& $4.21\pm0.11$ \\
        \bottomrule
    \end{tabular}
    }
    \caption{Performance comparison of Killian \textit{et al.} [2022] DDLPO topline, with and without ground truth transition probabilities as input features. Results are shown for evaluation on a single, fixed training seed. The results suggest that transition probability features do not significantly improve the final performance of the topline DDLPO model--this implies that PreFeRMAB does not leverage these features for an unfair reward advantage.}
    \label{table:ddlpo_withfeats}
\end{table}

In \autoref{table:ddlpo_withfeats}, we show that having access to {\bf features does not boost the performance of DDLPO}. Features help PreFeRMAB generalize to unseen arms and achieve strong zero-shot results, as demonstrated in the main paper. However, one may ask whether access to these features, as used by PreFeRMAB, may provide an unfair reward advantage over DDLPO, which in its original form \cite{killian2022restless} does not utilize feature information. That is, because input features in our experiments are derived from the original arm transition probabilities, it may be the case that these are used to achieve better performance. To determine whether there is an advantage from utilizing these features, we modify the original DDLPO model to accept ground truth transition probabilities for each arm as feature inputs to the respective policy networks. We present results for DDLPO with and without input features, for a fixed seed, in Table \autoref{table:ddlpo_withfeats}. In this table, we observe that across synthetic experiments for various system capacities and budgets, DDLPO's performance does not improve given access to features. These results suggest that PreFeRMAB is not leveraging the input features to gain an unfair advantage in evaluation.

\subsection{Different values of $N,B,S$}
\label{sec:appendix_different_n_b_s}
We present results on a wider range of problem settings, specifically different number of arms $N$, different budget $B$, and (for SIS Epidemic Modeling only), different number of possible states $S$. 

\begin{table}[htb!]
\scalebox{.75}{
\begin{tabular}{@{}lccccc@{}}
\toprule
$\frac{\textbf{Number of arms}}{\textbf{System capacity}}$ & \textbf{80\%} & \textbf{85\%} & \textbf{90\%} & \textbf{95\%} & \textbf{100\%}\\ 
\midrule
\multicolumn{6}{c}{System capacity $N=96$.  Budget $B=32$.} \\
\midrule
  No Action                  & $3.22_{\pm 0.12}$ & $3.24_{\pm 0.12}$ & $3.19_{\pm 0.11}$ & $3.18_{\pm 0.11}$ & $3.18_{\pm 0.11}$ \\
  Random Action              & $3.62_{\pm 0.13}$ & $3.66_{\pm 0.13}$ & $3.58_{\pm 0.13}$ & $3.60_{\pm 0.12}$ & $3.60_{\pm 0.12}$ \\
  {\bf PreFeRMAB}                   & $4.63_{\pm 0.12}$ & $4.71_{\pm 0.12}$ & $4.53_{\pm 0.13}$ & $4.47_{\pm 0.12}$ & $4.61_{\pm 0.10}$ \\
  DDLPO (topline)                & \textit{n/a} & \textit{n/a} & \textit{n/a} & \textit{n/a} & $4.58_{\pm 0.13}$ \\
  \midrule
\multicolumn{6}{c}{System capacity $N=48$.  Budget $B=16$.} \\
\midrule 
  No Action                  & $3.19_{\pm 0.11}$ & $3.24_{\pm 0.13}$ & $3.18_{\pm 0.11}$ & $3.23_{\pm 0.11}$ & $3.17_{\pm 0.10}$ \\
  Random Action              & $3.61_{\pm 0.17}$ & $3.70_{\pm 0.21}$ & $3.56_{\pm 0.18}$ & $3.67_{\pm 0.17}$ & $3.58_{\pm 0.16}$ \\
    {\bf PreFeRMAB}                   & $4.77_{\pm 0.18}$ & $4.74_{\pm 0.16}$ & $4.62_{\pm 0.19}$ & $4.94_{\pm 0.14}$ & $4.78_{\pm 0.14}$ \\
    DDLPO (topline)                & \textit{n/a} & \textit{n/a} & \textit{n/a} & \textit{n/a} &    $4.76_{\pm 0.14}$\\
  \midrule
\multicolumn{6}{c}{System capacity $N=21$. Budget $B=7$.} \\
\midrule
No Action                  & $3.44_{\pm 0.21}$ & $3.43_{\pm 0.19}$ & $3.41_{\pm 0.20}$ & $3.38_{\pm 0.17}$ & $3.22_{\pm 0.16}$ \\
  Random Action              & $3.82_{\pm 0.32}$ & $3.79_{\pm 0.33}$ & $3.77_{\pm 0.31}$ & $3.76_{\pm 0.28}$ & $3.58_{\pm 0.27}$ \\
  {\bf PreFeRMAB}                   & $4.20_{\pm 0.27}$ & $4.46_{\pm 0.23}$ & $4.48_{\pm 0.23}$ & $4.74_{\pm 0.26}$ & $4.56_{\pm 0.23}$ \\
  DDLPO (topline)                & \textit{n/a} & \textit{n/a} & \textit{n/a} & \textit{n/a} & $4.81_{\pm 0.14}$ \\
  \bottomrule
\end{tabular}
}
\caption{Results on Synthetic with discrete states. We present final reward divided by the number of arms, averaged over 50 trials. For each system capacity $N$, we pretrain a model. The DDLPO (topline) does not accomodate different opt-in rates and can only be used on 100\% opt-in.}
\label{table:discrete_state_synthetic}
\end{table}

\textbf{Synthetic Evaluation}: We first evaluate the performance of PreFeRMAB in the discrete state synthetic environment setting described above. Table~\ref{table:discrete_state_synthetic} illustrates these results. In this synthetic setting, we find that PreFeRMAB is able to consistently outperform Random Action and No Action baselines, and achieve performance comparable to the topline DDLPO approach. Critically, PreFeRMAB achieves good reward outcomes across changing system capacity $N$, budgets $B$, as well as different opt-in rates. Additionally, we find that PreFeRMAB achieves \textit{near-topline results from zero-shot learning} in the synthetic setting, compared to the topline DDLPO approach which is trained and evaluated on a fixed set of arm transition dynamics for 100 epochs (we take the best performance of DDLPO across the 100 epochs).

\begin{table}[htb!]
\scalebox{.75}{
\begin{tabular}{@{}lccccc@{}}
\toprule
$\frac{\textbf{Number of arms}}{\textbf{System capacity}}$ & \textbf{80\%} & \textbf{85\%} & \textbf{90\%} & \textbf{95\%} & \textbf{100\%}\\ \midrule
\multicolumn{5}{c}{Number of possible states per arm $S=150$.} \\
\midrule
  No Action & $5.33_{\pm 0.16}$ & $5.30_{\pm 0.15}$ & $5.31_{\pm 0.14}$ & $5.29_{\pm 0.13}$ & $5.28_{\pm 0.13}$\\
  Random Action & $7.03_{\pm 0.17}$ & $7.13_{\pm 0.16}$ & $7.02_{\pm 0.14}$ & $7.11_{\pm 0.13}$ & $7.06_{\pm 0.13}$\\
{\bf PreFeRMAB} & $8.35_{\pm 0.12}$ & $8.38_{\pm 0.11}$ & $8.26_{\pm 0.11}$ & $8.10_{\pm 0.11}$ & $8.00_{\pm 0.10}$\\
DDLPO (topline)                & \textit{n/a} & \textit{n/a} & \textit{n/a} & \textit{n/a} & $8.09_{\pm 0.11}$ \\
    \midrule
\multicolumn{5}{c}{Number of possible states per arm $S=100$.} \\
\midrule
  No Action & $5.28_{\pm 0.15}$ & $5.20_{\pm 0.13}$ & $5.30_{\pm 0.15}$ & $5.25_{\pm 0.14}$ & $5.27_{\pm 0.15}$\\
  Random Action & $6.95_{\pm 0.19}$ & $7.01_{\pm 0.16}$ & $7.11_{\pm 0.16}$ & $7.06_{\pm 0.15}$ & $7.07_{\pm 0.15}$\\
  {\bf PreFeRMAB} & $7.88_{\pm 0.20}$ & $7.91_{\pm 0.19}$ & $7.99_{\pm 0.18}$ & $8.01_{\pm 0.17}$ & $8.02_{\pm 0.16}$\\
  DDLPO (topline)                & \textit{n/a} & \textit{n/a} & \textit{n/a} & \textit{n/a} & $7.99_{\pm 0.08}$ \\
    \midrule
\multicolumn{5}{c}{Number of possible states per arm $S=50$.} \\
\midrule
  No Action & $5.39_{\pm 0.15}$ & $5.47_{\pm 0.15}$ & $5.42_{\pm 0.13}$ & $5.44_{\pm 0.12}$ & $5.46_{\pm 0.12}$\\
  Random Action & $7.29_{\pm 0.17}$ & $7.33_{\pm 0.17}$ & $7.26_{\pm 0.14}$ & $7.38_{\pm 0.15}$ & $7.33_{\pm 0.12}$\\
  {\bf PreFeRMAB} & $8.51_{\pm 0.08}$ & $8.37_{\pm 0.11}$ & $8.24_{\pm 0.07}$ & $8.10_{\pm 0.10}$ & $7.93_{\pm 0.09}$\\
    DDLPO (topline)                & \textit{n/a} & \textit{n/a} & \textit{n/a} & \textit{n/a} &  $8.04_{\pm 0.08}$\\
    \bottomrule
\end{tabular}
}
\caption{Results on SIS Epidemic Model with discrete states. We present final reward divided by the number of arms, averaged over 50 trials. System capacity $N=20$ and budget $B=16$. For each number of possible states per arm $S$, we pretrain a model. The DDLPO (topline) does not accomodate different opt-in rates and can only be used on 100\% opt-in.}
\label{table:discrete_state_sis}
\end{table}

\begin{table}[htb!]
\scalebox{.75}{
\begin{tabular}{@{}lccccc@{}}
\toprule
$\frac{\textbf{Number of arms}}{\textbf{System capacity}}$ & \textbf{80\%} & \textbf{85\%} & \textbf{90\%} & \textbf{95\%} & \textbf{100\%}\\ 
\midrule
\multicolumn{5}{c}{System capacity $N=25$. Budget $B=7$.} \\
\midrule
  No Action & $2.12_{\pm 0.26}$ & $2.30_{\pm 0.29}$ & $2.29_{\pm 0.27}$ & $2.19_{\pm 0.23}$ & $2.26_{\pm 0.25}$\\
  Random Action & $2.86_{\pm 0.32}$ & $3.27_{\pm 0.40}$ & $3.01_{\pm 0.30}$ & $3.09_{\pm 0.35}$ & $2.96_{\pm 0.31}$\\
  {\bf PreFeRMAB} & $5.06_{\pm 0.34}$ & $5.26_{\pm 0.33}$ & $4.68_{\pm 0.33}$ & $4.75_{\pm 0.35}$ & $4.61_{\pm 0.27}$\\
  DDLPO (topline)                & \textit{n/a} & \textit{n/a} & \textit{n/a} & \textit{n/a} & $4.68 _{\pm 0.09}$ \\
    \midrule
\multicolumn{5}{c}{System capacity $N=25$. Budget $B=5$.} \\
\midrule
  No Action & $2.14_{\pm 0.23}$ & $2.29_{\pm 0.26}$ & $2.24_{\pm 0.28}$ & $2.36_{\pm 0.24}$ & $2.19_{\pm 0.23}$\\
    Random Action & $2.68_{\pm 0.31}$ & $2.95_{\pm 0.36}$ & $2.75_{\pm 0.32}$ & $2.92_{\pm 0.26}$ & $2.69_{\pm 0.21}$\\
  {\bf PreFeRMAB} & $4.10_{\pm 0.32}$ & $4.45_{\pm 0.40}$ & $4.39_{\pm 0.33}$ & $4.48_{\pm 0.34}$ & $3.95_{\pm 0.34}$\\
  DDLPO (topline)                & \textit{n/a} & \textit{n/a} & \textit{n/a} & \textit{n/a} & $4.29 _{\pm 0.25}$\\
    \midrule
\multicolumn{5}{c}{System capacity $N=50$. Budget $B=10$.} \\
\midrule
  No Action & $2.27_{\pm 0.24}$ & $2.31_{\pm 0.17}$ & $2.19_{\pm 0.22}$ & $2.21_{\pm 0.21}$ & $2.27_{\pm 0.23}$\\
  Random Action & $2.82_{\pm 0.26}$ & $2.91_{\pm 0.22}$ & $2.72_{\pm 0.23}$ & $2.69_{\pm 0.18}$ & $2.77_{\pm 0.23}$\\
   {\bf PreFeRMAB} & $4.21_{\pm 0.30}$ & $3.98_{\pm 0.28}$ & $3.85_{\pm 0.28}$ & $3.68_{\pm 0.28}$ & $3.62_{\pm 0.26}$\\
   DDLPO (topline)                & \textit{n/a} & \textit{n/a} & \textit{n/a} & \textit{n/a} & $4.08 _{\pm 0.26}$ \\  \bottomrule
\end{tabular}
}
 \caption{Results on ARMMAN with discrete states. We present final reward divided by the number of arms, averaged over 50 trials. For each pair of $(N, B)$, we pretrain a model. The DDLPO (topline) does not accomodate different opt-in rates and can only be used on 100\% opt-in.}
\label{table:discrete_state_armman}
\end{table}

{\bf SIS Evaluation:} Next, we evaluate the performance of PreFeRMAB in the discrete-state SIS modelling setting. Table~\ref{table:discrete_state_sis} illustrates these results. We evaluate PreFeRMAB for $N=20, B=16$ on three different number of possible states per arm $S = 50, 100, 150$, representing the maximum population of a region in the SIS setting. The results shown demonstrate that PreFeRMAB performs well in \textit{zero-shot learning} in settings that model real-world planning problems, especially with larger state spaces and with multiple actions. We again find that PreFeRMAB achieves results comparable to the DDLPO topline with zero-shot testing, compared to DDLPO trained and evaluated on the same constant set of arms.  

{\bf ARMMAN Evaluation:} We next evaluate the performance of PreFeRMAB in the discrete state ARMMAN modeling setting. Table~\ref{table:discrete_state_armman} illustrates these results. In these experiments, we show performance for $S=3$ across 3 training configurations ($(N=25, B=5), (N=25, B=7), (N=50, B=10)$) for 5 test-time opt-in rates. We observe that our approach again performs consistently well in a more \textit{challenging} setting that models real-world planning problems across different system capacities, budgets, and opt-in rates. Specifically, we validate that PreFeRMAB can achieve higher average rewards for increased budgets given a fixed system capacity, which is expected as reward potential increases with higher budgets. Additionally, we see that PreFeRMAB again achieves \textit{zero-shot} results comparable to the DDLPO topline reward, reaching $\sim90\%$ of the topline reward in zero-shot evaluation.


\section{Multi-arm Generalization}
\label{sec:appendix_multi_arm_generalization}

In the main paper (Table~\ref{table:uniquesamples_vs_reward}), we presented results on Synthetic with $N=21, B=7$, demonstrating the benefit of multi-arm generalization. The results are obtained when the Wasserstein distance between training and testing distribution is 0.05 (see Sec~\ref{sec:appendix_dist_shift_details} for how we compute the Wasserstein distance). We provide additional results to further showcase the benefits of multi-arm generalization. Specifically, in Table~\ref{table:uniquesamples_vs_reward_nomask}, we present results for $N=12, B=3$. 

\begin{table}[h!]
\scalebox{.75}{
    \begin{tabular}{@{}lccccc@{}}
        \toprule
        \multicolumn{6}{c}{\textbf{System capacity $N=12$.  Budget $B=3$.}}  \\
        \# Unique arms & 48 & 39 & 30 & 21 & 12 \\ \midrule
        No Action       & $3.11_{\pm0.31}$& $3.11_{\pm0.31}$& $3.11_{\pm0.31}$& $3.11_{\pm0.31}$& $3.11_{\pm0.31}$\\
        Random Action   & $3.45_{\pm0.31}$& $3.45_{\pm0.31}$& $3.45_{\pm0.31}$& $3.45_{\pm0.31}$& $3.45_{\pm0.31}$\\
        {\bf PreFeRMAB} & $4.35_{\pm0.28}$& $4.28_{\pm0.29}$& $4.31_{\pm0.27}$& $4.04_{\pm0.32}$& $3.60_{\pm0.30}$\\
        \midrule
    \end{tabular}
    }
    \caption{Multi-arm generalization results on Synthetic (opt-in 100\%). With the same total amount of data, PreFeRMAB achieves stronger performance when pretrained on more unique arms.
} 
    \label{table:uniquesamples_vs_reward_nomask}
\end{table}


\section{Proof of Multi-Arm Generalization}
\label{sec:multiarm_gen_proof}
In this section, we will shorten $n_{\mathsf{epochs}}$ to $n$ for the sake of clarity. In this section, we let $C_{\mathsf{sys}}$ to denote a constant which depends on the parameters of the MDP such as budget per arm $B/N$, cost $c_j$, discount factor $\beta$, $\lambda_{\max}$, $R_{\max}$, $D$, $d$ and $L$. It can denote a different constant in every appearance. We list the assumptions made in the statement of the proposition below for the sake of clarity.
\begin{assumption}\label{as:opt_sample}
Suppose the learning algorithm learns neural network weights $\hat{\theta}$, whose policy is optimal for each $(\hat{\mu}_i,\lambda)$ for $i = 1,2,\dots,n$ and $\lambda \in [0,\lambda_{\max}]$. That is, it learns the optimal policy for every sample in the training data. 
\end{assumption}

\begin{assumption}\label{as:opt_as}
There exists a choice of weights $\theta^* \in \Theta$ which gives the optimal policy for every set of $N$ features $(\hat{\mu})$ drawn as the empirical distribution of i.i.d. samples from $\mu^*$ and for every $\lambda \in [0,\lambda_{\max}]$
\end{assumption}

\begin{assumption}\label{as:lipschitz_value}
$\Theta = \mathcal{B}_2(D,\mathbb{R}^d)$, the $\ell_2$ ball of radius $D$ in $\mathbb{R}^d$. We assume that $$|V(\vs,\theta_1,\lambda,\hat{\mu})-V(\vs,\theta_2,\lambda,\hat{\mu})|\leq L\|\theta_1-\theta_2\|$$

$$|V(\vs,\theta,\lambda_1,\hat{\mu})-V(\vs,\theta,\lambda_2,\hat{\mu})|\leq L|\lambda_1-\lambda_2|$$
\end{assumption}
Define the population average value function by $\bar{V}(s,\theta) = \bE_{\hat{\mu}}\inf_{\lambda \in [0,\lambda_{\max}]}V(\vs,\theta,\lambda,\hat{\mu})$ and the sample average value function by $\hat{V}(s,\theta) = \frac{1}{n}\sum_{j=1}^{n}\inf_{\lambda \in [0,\lambda_{\max}]}V(\vs,\theta,\lambda,\hat{\mu}) $
\sloppy

 Now, consider:

\begin{align}
    \bar{V}(s,\hat{\theta})-\bar{V}(s,\theta^{*}) &= \bar{V}(\vs,\hat{\theta})-\hat{V}(\vs,\hat{\theta})+\hat{V}(\vs,\hat{\theta})-\hat{V}(\vs,\theta^*)\nonumber\\ &\quad+\hat{V}(\vs,\theta^*)-\bar{V}(\vs,\theta^*) \nonumber \\
&= \bar{V}(\vs,\hat{\theta})-\hat{V}(\vs,\hat{\theta})+\hat{V}(\vs,\theta^*)-\bar{V}(\vs,\theta^*) \nonumber \\
&\geq -2\sup_{\theta \in \Theta}|\bar{V}(\vs,\theta)-\hat{V}(\vs,\theta)|
\end{align}
 
The first step follows by adding and subtracting the same term. In the second step, we have used the fact that Assumptions~\ref{as:opt_sample} and~\ref{as:opt_as} imply that $\hat{V}(\vs,\hat{\theta})= \hat{V}(\vs,\theta^*)$. In the third step, we have replaced the discrepancy between the sample averate and the population average at specific points $\hat{\theta},\theta^*$ with the uniform bound over the parameter set $\Theta$. 

We use the Rademacher complexity bounds to bound this term. By \cite[Lemma 26.2]{shalev2014understanding}, we show the following:

Let $S$ denote the random training sample $(\hat{\mu}_1,\dots,\hat{\mu}_n)$ and $P_0$ denote the uniform distribution $\mathsf{Unif}(\{-1,1\}^n)$. Then, for some numerical constant $C$, we have:

$$\bE_{S}\sup_{\theta \in \Theta}|\bar{V}(\vs,\theta)-\hat{V}(\vs,\theta)| \leq C \bE_{S}\mathcal{R}(\Theta\circ S)$$

Where, $\mathcal{R}(\Theta \circ S)$ is the Rademacher complexity:

\begin{align}&\mathcal{R}(\Theta \circ S) := \nonumber \\&
\frac{1}{n}\bE_{\mathbf{\sigma}\sim P_0}\sup_{\theta \in \Theta}\sum_{i=1}^{n} \sigma_i[\inf_{\lambda}V(\vs,\theta,\lambda,\hat{\mu}_i) - \bE_{\hat{\mu}}\inf_{\lambda}V(\vs,\theta,\lambda,\hat{\mu})]\nonumber\end{align}

Thus, to demonstrate the result, it is sufficient to show that: 
\begin{equation}\label{eq:desired_bound}
\mathcal{R}(\Theta \circ S) \leq \frac{C_{\mathsf{sys}}\mathsf{polylog}(Nn)}{\sqrt{nN}}
\end{equation}

We will dedicate the rest of this section to demonstrate Equation~\eqref{eq:desired_bound}. First we will state a useful Lemma which follows from \cite[Lemma 1.2.1]{vershynin2018high}

\begin{lemma}\label{lem:int_tail}
Suppose a positive random variable $X$ satisfies:
$\mathbb{P}(X > t) \leq A\exp(-\frac{t^2}{2B})$ for some $B > 0$, $A > e$ and for every $t \geq 0$ then for some numerical constant $C$, we have:

$$\bE[X] \leq C \sqrt{B\log A}$$
\end{lemma}
\begin{proof}
    From \cite[Lemma 1.2.1]{vershynin2018high}, we have:
$\bE X = \int_0^{\infty} \mathbb{P}(X > t)dt$. Thus, we conclude:
\begin{align}
    \bE X &\leq \int_{0}^{\infty} \min(1,A\exp(-\frac{t^2}{2B}))dt \nonumber \\
&= \sqrt{2B\log A} + \int_{\sqrt{2B\log A}}^{\infty}A\exp(-\frac{t^2}{2B})dt \nonumber \\
&= \sqrt{2B\log A} + \int_{0}^{\infty}A\exp(-\tfrac{(t+\sqrt{2B\log A})^2}{2B})dt \nonumber \\
&\leq \sqrt{2B\log A} + \int_{0}^{\infty}\exp(-\tfrac{t^2}{2B})dt \nonumber \\
&\leq \sqrt{2B\log A} + \sqrt{2\pi B}
\end{align}

In the fourth step we have used the fact that $\exp(-(a+b)^2) \leq \exp(-a^2 - b^2)$ whenever $a,b > 0$. 
\end{proof}

Define 
$$v_i(\theta) := [\inf_{\lambda}V(\vs,\theta,\lambda,\hat{\mu}_i) - \bE_{\hat{\mu}}\inf_{\lambda}V(\vs,\theta,\lambda,\hat{\mu})]\,.$$
We have the following lemma controlling how large $v_i$ is for any given $\theta$.

\begin{lemma}\label{lem:distribution_variation}
For any $\delta > 0$, with probability at-least $1-\delta$, $$\sup_i|v_i(\theta)|\leq \sqrt{\frac{C_{\mathsf{sys}}\log(\frac{Nn}{\delta})}{N}}$$

Where $C_{\mathsf{sys}}$ depends on the system parameters.
\end{lemma}

\begin{proof}
   First, we note that:
\begin{align}
    &|\inf_{\lambda}V(\vs,\theta,\lambda,\hat{\mu}_i) - \bE_{\hat{\mu}}\inf_{\lambda}V(\vs,\theta,\lambda,\hat{\mu})| \nonumber \\&\leq \bE_{\hat{\mu}}|\inf_{\lambda}V(\vs,\theta,\lambda,\hat{\mu}_i) - \inf_{\lambda}V(\vs,\theta,\lambda,\hat{\mu})| \nonumber \\
&\leq \bE_{\hat{\mu}}\sup_{\lambda \in [0,\lambda_{\max}]}|V(\vs,\theta,\lambda,\hat{\mu}_i)-V(\vs,\theta,\lambda,\hat{\mu})| \nonumber \\
&\leq \sup_{\lambda \in [0,\lambda_{\max}]}|V(\vs,\theta,\lambda,\hat{\mu}_i)-\bE[V(\vs,\theta,\lambda,\hat{\mu}_i)]| \nonumber\\ &\quad+ \bE_{\hat{\mu}}\sup_{\lambda \in [0,\lambda_{\max}]}|V(\vs,\theta,\lambda,\hat{\mu})-\bE[V(\vs,\theta,\lambda,\hat{\mu})]|  
\end{align}
In the last step, we have used the fact that $\hat{\mu}$ and $\hat{\mu}_i$ are identically distributed and hence $\bE[V(\vs,\theta,\lambda,\hat{\mu})] = \bE[V(\vs,\theta,\lambda,\hat{\mu}_i)]$.
Note that by definition, the value function $V(s,\theta,\lambda,\hat{\mu}_i) = \frac{1}{N}\sum_{j=1}^{N}V(s_j,\theta,\lambda,\vz_j)$. Thus, it is clear that $|V(s,\theta,\lambda,\hat{\mu}_i)|\leq A\frac{1+\lambda_{\max}}{(1-\beta)} =: V_{\max}$ where $A$ is a constant which depends on the cost parameters $c_j, \frac{B}{N}$ and the maximum reward. Take $\hat{\mu}_i := (\vz_1^{(i)},\dots,\vz_N^{(i)})$ and $\hat{\mu} := (\vz_1,\dots,\vz_N)$. 

Thus, for a given $\lambda$, we have: $V(\vs,\theta,\lambda,\hat{\mu}_i) - \bE V(\vs,\theta,\lambda,\hat{\mu}_i)$ has zero mean and 
\begin{align}
    &V(\vs,\theta,\lambda,\hat{\mu}_i) - \bE V(\vs,\theta,\lambda,\hat{\mu}_i) \nonumber \\ &= \frac{1}{N}\sum_{j=1}^{N}[V(s_j,\theta,\lambda,\vz_j^{(i)})-\bE V(s_j,\theta,\lambda,\vz_j^{(i)})]
\end{align}

It is an average of $N$ i.i.d. zero mean random variables, bounded almost surely by $2V_{\max}$. Therefore, using the Azuma-Hoeffding inequality (\cite{vershynin2018high}), we have:
$$\mathbb{P}\left(|V(\vs,\theta,\lambda,\hat{\mu}_i) - \bE V(\vs,\theta,\lambda,\hat{\mu}_i)| > t\right) \leq C \exp(-\tfrac{c_1 Nt^2}{V_{\max}^2})$$

Only in this proof, let $|V(\vs,\theta,\lambda,\hat{\mu}_i) - \bE V(\vs,\theta,\lambda,\hat{\mu}_i)| =: H(\lambda)$ for the sake of clarity. Let $B \subseteq [0,\lambda_{\max}]$ be any finite subset. Then, by union bound, we have:

\begin{equation}\label{eq:ub_1}\mathbb{P}\left(\sup_{\lambda \in B}H(\lambda) > t\right) \leq C |B|\exp(-\tfrac{c_1 Nt^2}{V_{\max}^2})\end{equation}

 Suppose $B$ is an $\epsilon$-net for the set $[0,\lambda_{\max}]$ for some $\epsilon >0$. This can be achieved with $|B| = \frac{\lambda_{\max}}{\epsilon}$. Let $f: [0,\lambda] \to B$ map $\lambda$ to the closest element in $B$

  \begin{align}
      &\sup_{\lambda \in [0,\lambda_{\max}]} H(\lambda) = \sup_{\lambda \in [0,\lambda_{\max}]} H(f(\lambda)) + H(\lambda)-H(f(\lambda)) \nonumber \\
&\leq \sup_{\lambda \in [0,\lambda_{\max}]} H(f(\lambda)) + 2L\epsilon \nonumber \\
&\leq \sup_{\lambda \in B} H(\lambda) + 2L\epsilon
  \end{align}
Taking $\epsilon = \frac{1}{\sqrt{N}}$, we conclude from Equation~\eqref{eq:ub_1} that with probability at-least $1-\delta$:

$$\sup_{\lambda \in [0,\lambda_{\max}]} H(\lambda) \leq C_{\mathsf{sys}} \sqrt{\frac{\log(\frac{N}{\delta})}{N}}$$

The same concentration bounds hold for $\sup_{\lambda}|V(\vs,\theta,\lambda,\hat{\mu}) - \bE V(\vs,\theta,\lambda,\hat{\mu})|$ and integrating the tails (Lemma \ref{lem:int_tail}), we bound obtain the bound:
$$\sup_{\lambda \in [0,\lambda_{\max}]}|V(\vs,\theta,\lambda,\hat{\mu}) - \bE V(\vs,\theta,\lambda,\hat{\mu})| \leq C_
{\mathsf{sys}} \sqrt{\frac{\log N}{N}}$$

Applying a union bound over $i = 1,\dots, n$, conclude the result.
\end{proof}

We state the following folklore result regarding concentration of i.i.d. Rademacher random variables.
\begin{lemma}\label{lem:rademacher_conc}
 Given constants $a_1,\dots,a_n \in \mathbb{R}$, and $\sigma_1,\dots,\sigma_n$ i.i.d Rademacher random variables, then for any $\delta > 0$, we have with probability at-least $1-\delta$:
$$\sum_{i=1}^{n}\sigma_i a_i \leq C\sqrt{\sum_i a_i^2}\sqrt{\log(\tfrac{1}{\delta})}$$

Where $C$ is a numerical constant
\end{lemma}

We are now ready to prove Equation~\eqref{eq:desired_bound} and hence complete the proof of Proposition~\ref{prop:multiarm_gen}. Given a data set $\hat{\mu}_1,\dots,\hat{\mu}_n$ and $\theta \in \Theta$, we let $v_i(\theta) := [\inf_{\lambda}V(\vs,\theta,\lambda,\hat{\mu}_i) - \bE_{\hat{\mu}}\inf_{\lambda}V(\vs,\theta,\lambda,\hat{\mu})]$.
Given a finite set $\hat{\Theta} := \{\theta_1,\dots,\theta_H\}\subseteq \Theta$, from Lemma~\ref{lem:distribution_variation}, we have with probability $1-\delta$, 

$$\sup_{\theta \in \hat{\Theta}}\sup_i|v_{i}(\theta)|\leq \sqrt{\frac{C_{\mathsf{sys}}\log(\frac{nN|\hat{\Theta}|}{\delta})}{N}} =: R(\delta)$$

Therefore, with probability at-least $1-\delta$ over the randomness in $\hat{\mu}_1,\dots,\hat{\mu}_n$, we have:

\begin{align}
&\mathbb{P}\left(\sup_{\theta \in \hat{\Theta}}\sum_{i=1}^{n}\sigma_i v_i(\theta) > t\bigr|\hat{\mu}_1,\dots,\hat{\mu}_n\right) \nonumber \\ &\leq C_1 |\hat{\Theta}|\exp\left(-\frac{C_2t^2}{n R^2(\delta)}\right)\label{eq:hp_bound}
\end{align}

We pick $\hat{\Theta}$ to be an $\epsilon$ net over $\Theta$. By \cite[Corollary 4.2.13]{vershynin2018high}, we can take $|\hat{\Theta}| \leq (\frac{3D}{\epsilon})^d$. Let $f:\Theta \to \hat{\Theta}$ be the map to its nearest element in $\hat{\Theta}$. Now, we have:
\begin{align}
\sup_{\theta \in \Theta} \sum_i v_i(\theta)\sigma_i &= \sup_{\theta \in \Theta} \sum_i v_i(f(\theta))\sigma_i + [v_i(\theta)-v_i(f(\theta))]\sigma_i \nonumber \\
&\leq 2n\epsilon L + \sup_{\theta \in \hat{\Theta}} \sum_i v_i(\hat{\theta})\sigma_i \nonumber 
\end{align}

Combining this with Equation~\eqref{eq:hp_bound}, we conclude that with $1-\delta$ over the randomness in $\hat{\mu}_1,\dots,\hat{\mu}_n$, we have:

\begin{align}&\mathbb{P}\left(\sup_{\theta \in \Theta}\sum_{i=1}^{n}\sigma_i v_i(\theta) > t + 2n\epsilon L\biggr|\hat{\mu}_1,\dots,\hat{\mu}_n\right) \nonumber \\&\leq C_1 |\hat{\Theta}|\exp\left(-\frac{C_2t^2}{n R^2(\delta)}\right)\label{eq:hp_bound_2}
\end{align}

Taking $\epsilon = \frac{1}{n^{\frac{3}{2}}\sqrt{N}}$ and integrating the tails (Lemma~\ref{lem:int_tail}), we conclude that with probability at-least $1-\delta$ (with respect to the randomness in $\hat{\mu}_1,\dots,\hat{\mu}_N$).

$$\bE[\sup_{\theta \in \Theta}\sum_{i=1}^{n}\sigma_i v_i(\theta)|\hat{\mu}_1,\dots,\hat{\mu}_n] \leq C_{\mathsf{sys}}\frac{R(\delta)}{\sqrt{n}}\mathsf{polylog}(Nn)$$

Define the random variable $$X := \bE[\sup_{\theta \in \Theta}\sum_{i=1}^{n}\sigma_i v_i(\theta)|\hat{\mu}_1,\dots,\hat{\mu}_n]$$
Using the definition of $R(\delta)$, we have:

$$\mathbb{P}(X > t) \leq C_1\exp(-\frac{t^2 nN}{C_{\mathsf{sys}}\mathsf{polylog}(Nn)})\,.$$ 

We then apply Lemma~\ref{lem:int_tail} to the equation above to bound $\bE X$ and conclude Equation~\eqref{eq:desired_bound}.

\section{Proof for $\lambda$-network Update Rule and Convergence}
\label{sec:appendix_proofs}

\begin{proof}[Proof of Proposition~\ref{prop:lambda_updating_rule}]
We first consider a simple setting, where the opt-in and opt-out decisions of arms are fixed before training. Taking the derivative of the objective (Eq~\ref{eq:lagrangian_relaxation}) with respect to $\lambda$, we obtain:
\begin{align*}
\frac{B}{1-\beta} - 
\sum_{i=1}^N  \mathbb{E}\left[\sum_{\substack{t\in[H] \\ \text{arm $i$ opts-in at $t$}}} \beta^t c_{i,t} + \sum_{\substack{ t\in[H] \\ \text{arm $i$ opts-out at $t$}}} \beta^t c_{0,t}
\right].
\end{align*}
Now consider the general case that the opt-in and opt-out decisions are updated at each round during the training. We have 
\begin{align*}
&\Lambda_t = \Lambda_{t-1}  - \alpha\left(\frac{B}{1-\beta}\right) \\ 
& + \alpha\left(\sum_{i=1}^N  \mathbb{E}\left[\sum_{t=0}^H \mathbb{I}\{\xi_{i,t} = 1\} \beta^t c_{i,t} + \mathbb{I}\{\xi_{i,t} = 0\}\beta^tc_{0,t}
\right]\right),
\end{align*}
where the expectation is over the random variables $\xi_{i,t}$ and the action chosen by the optimal policy. Rearranging and simplifying the right hand side terms, we obtain the $\lambda$-updating rule. 
\end{proof}

\begin{proof}[Proof of Proposition~\ref{prop:convergence_lambda}]
The proof largely follows the proof of Proposition 2 in Killian \textit{et al.}\shortcite{killian2022restless}.

Since the max of piece-wise linear functions is a convex function, Equation~\ref{eq:lagrangian_relaxation} is convex in $\lambda$. Thus, it suffices to show (1) the gradient estimated using Proposition~\ref{prop:lambda_updating_rule} is accurate and (2) all inputs (states, features, opt-in decisions) are seen infinitely often in the limit. For (1), we note that training the policy network for a sufficient number of epochs under a fixed output of the $\lambda$-network ensures that Q-value estimates are accurate. With accurate Q-functions and corresponding optimal policies, the sampled cumulative sum of action costs is an unbiased estimator of expected cumulative sum of action costs. Critically, for the estimator to be unbiased, we do not strictly enforce the budget constraint during training, as in Killian \textit{et al.}\shortcite{killian2022restless}. In inference, we do strictly enforce the budget constraint. For (2), we note that during training, initial states are uniformly sampled, and opt-in decisions are also sampled from a fixed bernoulli distribution.For arms that newly opt-in, the features are uniformly sampled. Thus, both (1) and (2) are achieved.
\end{proof}

\end{document}